\documentclass{article}

\usepackage{microtype}
\usepackage{graphicx}
\usepackage{subfigure}
\usepackage{booktabs} % for professional tables

\usepackage{hyperref}

\usepackage[accepted]{icml2020}

\usepackage{enumerate}
\usepackage{amsfonts}
\usepackage{amsmath}
\usepackage{amsthm}
\usepackage{amssymb}
\usepackage{verbatim}
\usepackage{mathrsfs}
\usepackage{multicol}
\usepackage{pifont}

\usepackage{tikz}
\usetikzlibrary{decorations.markings}
\usetikzlibrary{arrows}
\tikzset{>=latex}

\newcommand{\cmark}{\ding{51}}%
\newcommand{\bcmark}{\ding{52}}%
\newcommand{\xmark}{\ding{55}}%

\DeclareMathOperator{\D}{\mathbb{D}}
\DeclareMathOperator{\M}{\mathbb{M}}
\DeclareMathOperator{\DL2}{\mathscr{D}_{L^2}}
\DeclareMathOperator{\hD}{\hat{\D}}

\theoremstyle{plain}
\newtheorem{thm}{Theorem}%[section]

\newtheorem{lemma}[thm]{Lemma}

\usepackage{color}

\icmltitlerunning{Fuzzy c-Means Clustering for Persistence Diagrams}

\begin{document}

\twocolumn[

\icmltitle{Fuzzy c-Means Clustering for Persistence Diagrams}

% It is OKAY to include author information, even for blind
% submissions: the style file will automatically remove it for you
% unless you've provided the [accepted] option to the icml2020
% package.

% List of affiliations: The first argument should be a (short)
% identifier you will use later to specify author affiliations
% Academic affiliations should list Department, University, City, Region, Country
% Industry affiliations should list Company, City, Region, Country

% You can specify symbols, otherwise they are numbered in order.
% Ideally, you should not use this facility. Affiliations will be numbered
% in order of appearance and this is the preferred way.
\icmlsetsymbol{equal}{*}

\begin{icmlauthorlist}
\icmlauthor{Thomas Davies}{soton}
\icmlauthor{Jack Aspinall}{ox}
\icmlauthor{Bryan Wilder}{har}
\icmlauthor{Long Tran-Thanh}{war}
\end{icmlauthorlist}

\icmlaffiliation{soton}{Electronics and Computer Science, University of Southampton, UK}
\icmlaffiliation{ox}{Department of Materials, University of Oxford, UK.}
\icmlaffiliation{har}{Center for Research on Computation and Society, John A. Paulson School of Engineering and Applied Sciences, Harvard University, USA.}
\icmlaffiliation{war}{Department of Computer Science, University of Warwick, UK.}

\icmlcorrespondingauthor{Thomas Davies}{t.o.m.davies@soton.ac.uk}

% You may provide any keywords that you
% find helpful for describing your paper; these are used to populate
% the "keywords" metadata in the PDF but will not be shown in the document
\icmlkeywords{Topological Data Analysis, Fuzzy Clustering, Machine Learning}

\vskip 0.3in
]

% this must go after the closing bracket ] following \twocolumn[ ...

% This command actually creates the footnote in the first column
% listing the affiliations and the copyright notice.
% The command takes one argument, which is text to display at the start of the footnote.
% The \icmlEqualContribution command is standard text for equal contribution.
% Remove it (just {}) if you do not need this facility.

\printAffiliationsAndNotice{}  % leave blank if no need to mention equal contribution
%\printAffiliationsAndNotice{\icmlEqualContribution} % otherwise use the standard text.

\begin{abstract}
Persistence diagrams concisely represent the topology of a point cloud whilst having strong theoretical guarantees, but the question of how to best integrate this information into machine learning workflows remains open. In this paper we extend the ubiquitous Fuzzy c-Means (FCM) clustering algorithm to the space of persistence diagrams, enabling unsupervised learning that automatically captures the topological structure of data without the topological prior knowledge or additional processing of persistence diagrams that many other techniques require. We give theoretical convergence guarantees that correspond to the Euclidean case, and empirically demonstrate the capability of our algorithm to capture topological information via the fuzzy RAND index. We end with experiments on two datasets that utilise both the topological and fuzzy nature of our algorithm: pre-trained model selection in machine learning and lattices structures from materials science. As pre-trained models can perform well on multiple tasks, selecting the best model is a naturally fuzzy problem; we show that fuzzy clustering persistence diagrams allows for model selection using the topology of decision boundaries. In materials science, we classify transformed lattice structure datasets for the first time, whilst the probabilistic membership values let us rank candidate lattices in a scenario where further investigation requires expensive laboratory time and expertise.
\end{abstract}

\section{Introduction}

% WHAT WE DO

Persistence diagrams, a concise representation of the topology of a point cloud with strong theoretical guarantees, have emerged as a new tool in the field of data analysis \citep{comptopbook}. Persistence diagrams have been successfully used to analyse problems ranging from identifying financial crashes \citep{crash} to analysing protein bindings \citep{protein}, but how to best integrate them into machine learning workflows remains uncertain due to the non-Hilbertian nature of persistence diagram space. Current efforts can be roughly split into three categories: (i) embedding persistence-diagrams into a Hilbert space which can be directly integrated into current machine learning and statistical techniques; (ii) adding topological error terms to loss functions; and (iii) unsupervised methods directly on the space of persistence diagrams. In this paper, we contribute to the latter category by extending the Euclidean Fuzzy c-Means (FCM) clustering algorithm \citep{BezdekPaper} to the space of persistence diagrams. It is widely accepted that many real-world datasets are not clearly delineated into hard categories \citep{frand}. Thus any algorithm that accounts for this is desirable, as evidenced by the large number of publications studying and using fuzzy clustering algorithms \citep{fuzzyreview, ex1, fuzzy2}. Our algorithm enables practitioners to study the fuzzy nature of data through a topological lens for the first time.

We perform the convergence analysis for our algorithm, giving the same guarantees as traditional FCM clustering: that every convergent subsequence of iterates tends to a local minimum or saddle point. As the space of persistence diagrams has far weaker theoretical properties than Euclidean space, this extension is non-trivial: see Section 1.2 and Appendix A for more details. As this guarantee could lead to non-convergence in practice, we empirically evaluate convergence on a total of 825 randomly generated persistence diagrams and find that our algorithm converges every time. We evaluate our algorithm using a variety of distances on persistence diagrams with the fuzzy RAND index \citep{frand}, a standard measure of cluster quality. We find that we fall into the standard paradigm, whereby distances that take longer to compute result in higher quality clustering, whereas approximations result in lower quality clusters but can be computed faster. By enabling the use of these distances, we can take advantage of the computational speed ups of vectorisations or approximate distances, whilst still having an end product in persistence diagram space. 

We demonstrate the practical value of our fuzzy clustering algorithm by using it to cluster datasets that benefit from both the topological and fuzzy nature of our algorithm. We apply our technique in two settings: (i) lattice structures in materials science; and (ii) pre-trained model selection. The motivation of the former is as follows: A key property for machine learning in materials science has been identified as ``invariance to the basis symmetries of physics [...] rotation, reflection, translation'' \citep{Schmidt2019}. Geometric clustering algorithms do not have this invariance, but persistence diagrams do, making them ideally suited for this application; we can cluster transformed lattice structure datasets where geometric equivalents fail. In addition to this, our probabilistic membership values allow us to rank the top-$k$ most likely lattices assigned to a cluster. This is particularly important in materials science, as further investigation requires expensive laboratory time and expertise. Our second application is inspired by previous work showing that pre-trained deep learning models perform better on tasks if they have topologically similar decision boundaries \citep{tda_db}. As one model can perform well on multiple tasks, this is a naturally fuzzy problem, and so ideally suited to our algorithm. We use our method to cluster models and tasks by the persistence diagrams of their decision boundaries. Not only is our algorithm able to successfully cluster models with the correct task, based just on the topology of its decision boundary, but we show that higher membership values imply better performance on unseen tasks. This further uncovers the link between the topology of decision boundaries and the generalisation power of machine learning.

\subsection{Related work}

\textbf{Means of persistence diagrams.} Our work relies on the existence of statistics in the space of persistence diagrams. \citet{milyenko} first showed that means and expectations are well-defined in the space of persistence diagrams. Specifically, they showed that the Fréchet mean, an extension of means onto metric spaces, is well-defined under weak assumptions on the space of persistence diagrams. \citet{Turner2012FrchetMF} then developed an algorithm to compute the Fréchet mean. 
%We adapt the algorithm by Turner et al. to the weighted case, but the combinatoric nature of their algorithm makes it computationally intense.
%\blue{However, due to the combinatoric nature of their algorithm, it is computationally intense, and therefore it is not trivial how it can be adapted to our framework in a computationally efficient way.
%Note that there is a relevant line of research for speeding up the computation of means and barycentres. 
%In particular, } 
However, the combinatoric nature of their algorithm makes it computationally intense. There is a relevant line of research for speeding up the computation of means and barycentres. 
In particular, \citet{lacombe} framed the computation of means and barycentres in the space of persistence diagram as an optimal transport problem, allowing them to use the Sinkhorn algorithm \cite{sinkhorn} for fast computation of approximate solutions. Techniques to speed up the underlying matching problem fundamental to our computation have also been proposed by \citet{pb} and \citet{geompd}. Our fuzzy clustering algorithm can integrate these solutions to further speed up its computing time if necessary.

\textbf{Learning with persistence-based summaries.} Integrating diagrams into machine learning workflows remained challenging even with well-defined means, as the space is non-Hilbertian \citep{pdisometry}. As such, efforts have been made to map diagrams into a Hilbert space; primarily either by embedding into finite feature vectors \citep{tropical, fabio1, persimages} or functional summaries \citep{persland, rieck3}, or by defining a positive-definite kernel on diagram space \citep{stabkern, slicekern, fishkern}. These vectorisations have been integrated into deep learning either by learning parameters for the embedding \citep{hofer1,perslay,landlay,learnweight, persbag}, or as part of a topological loss or regulariser \citep{chenreg, toplayer, seg5, topauto}. However, the embeddings used in these techniques deform the metric structure of persistence diagram space \citep{dist1, dist2, bilin}, potentially leading to the loss of important information. In comparison, our algorithm acts in the space of persistence diagrams so it does not deform the structure of diagram space via embeddings. However, as an unsupervised learning algorithm, our algorithm is intended to complement these techniques, offering a different approach for practitioners to use, rather than directly competing with them. Our algorithm offers value when practitioners require: fuzzy topological data processing; to stay in the space of persistence diagrams (e.g., to allow for a series of operations on persistence diagrams to take place); or unsupervised learning with persistence diagrams in a continuous way.

% Furthermore, these techniques generally require prior knowledge of a `correct' target topology which cannot plausibly be known in most scenarios. In comparison, our algorithm acts in the space of persistence diagrams so it does not deform the structure of diagram space via embeddings, and is entirely unsupervised, requiring no prior knowledge about the topology.

\textbf{Hard clustering.} \citet{inproceedings} gave an algorithm for hard clustering persistence diagrams based on the algorithm by \citet{Turner2012FrchetMF}. As mentioned earlier, many real-world datasets are not clearly delineated into hard categories \citep{frand}, and so a fuzzy algorithm would naturally be chosen over a hard clustering algorithm when dealing with such datasets. However, we believe our work offers further advantages over hard clustering:

\begin{enumerate}[(i)]
    \item Fuzzy membership values have been shown to be more robust to noise than discrete labels \citep{fuzzy}.
    \item Unlike hard clustering, fuzzy clustering is analytically differentiable, allowing integration of the fuzzy clustering step into deep learning methods \citep{bryan}.
    \item To the best of our knowledge, our implementation is the only publicly available clustering algorithm for persistence diagrams.
\end{enumerate}

\begin{figure*}[t!]
    \centering
    \includegraphics[width=\textwidth]{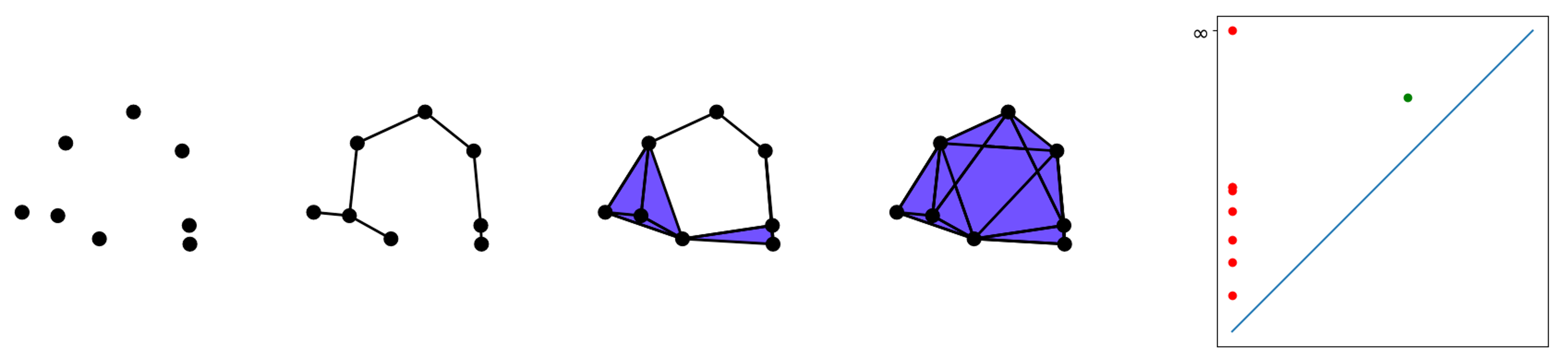}
    \caption{On the left is a Vietoris-Rips filtration, shown with increasing values for $\epsilon$. The convex hull of a subset $X$ of points is included when all points in $X$ are pairwise within $\epsilon$ of each other. On the right is the persistence diagram of the filtration. Each red point is in the 0-persistence diagram, and represents a connected component. The green point is in the 1-persistence diagram and represents the hole.}
\end{figure*}

\textbf{Geometric equivalents.} The most similar unsupervised learning technique to our algorithm is Wasserstein Barycentre Clustering (WBC). We compare our algorithm experimentally to WBC using ADMM \citep{ADMM}, Bregman ADMM \citep{BADDM}, Subgradient Descent \citep{sinkhorn}, Iterative Bregman Projection \citep{IBP}, and full linear programming \citep{linprog}. WBC clusters datasets of point clouds by the Wasserstein distance between the point clouds, rather than the Wasserstein distance between their persistence diagrams. Fuzzy discrete distribution clustering \citep{fwbc} offers similarities to our algorithm, but the addition of the diagonal with infinite multiplicity in persistence diagram space makes the theory distinct.
%We compare our algorithm experimentally to WBC using ADMM \citep{ADMM}, Bregman ADMM \citep{BADDM}, Subgradient Descent \citep{sinkhorn}, Iterative Bregman Projection \citep{IBP}, and full linear programming \citep{linprog}. Each of these algorithms computes or approximates the Wasserstein barycentre in different ways. 

\subsection{Our contributions}

\par 1. Our main contribution is an algorithm for Fuzzy c-Means clustering of persistence diagrams, along with the convergence analysis. Given a collection of persistence diagrams $\mathbb{D}_1,\dots, \mathbb{D}_n$, we alternatively calculate cluster centres $\mathbb{M}_1,\dots, \mathbb{M}_c$ and membership values $r_{jk} \in [0,1]$ which denote the degree to which diagram $\mathbb{D}_j$ is associated with cluster $\mathbb{M}_k$. We prove Theorem~\ref{thm: local convergence}, showing that every convergent subsequence of these alternative update steps tends to a local minimum or saddle point of the cost function. This is the same convergence guarantee provided by traditional FCM clustering \citep{repairs}, but requires significant additional work as the space of persistence diagrams with the Wasserstein distance has far weaker theoretical properties than Euclidean space.

\par 2. Updating the cluster centres requires computing the weighted Fréchet mean. We extend the algorithm given by \citet{Turner2012FrchetMF} to the weighted case, justifying our addition of weights by extending their proof to show that the updated algorithm converges. As well as enabling the fuzzy clustering, this gives the first algorithm to compute weighted barycentres of persistence diagrams.

\par 3. We evaluate the empirical behaviour of our algorithm, and apply it to two datasets. Our work in materials science uses persistence diagrams to solve a known issue with machine learning in materials science, whilst clustering the persistence diagrams of decision boundaries gives further evidence that the topology of the learnt decision boundaries provides valuable information about the generalisation power of machine learning algorithms. Moreover, it gives a working example of clustering many pre-trained models and tasks together, allocating model to task via the learnt cluster centres.

\par 4. We implement our algorithm in Python, available on GitHub\footnote{\url{https://github.com/tomogwen/fpdcluster}}. It works with persistence diagrams from commonly used open-source libraries for Topological Data Analysis (TDA),\footnote{\href{https://github.com/mrzv/dionysus}{Dionysus} and \href{https://ripser.scikit-tda.org/}{Ripser} \citep{ripser}.} so is available for easy integration into current workflows, offering a powerful unsupervised learning algorithm to data science practitioners using TDA. As far as we can tell, this is the only clustering algorithm for persistence diagrams that is publicly available.

\begin{figure*}[ht!]
\centering
  \includegraphics[width=0.9\textwidth]{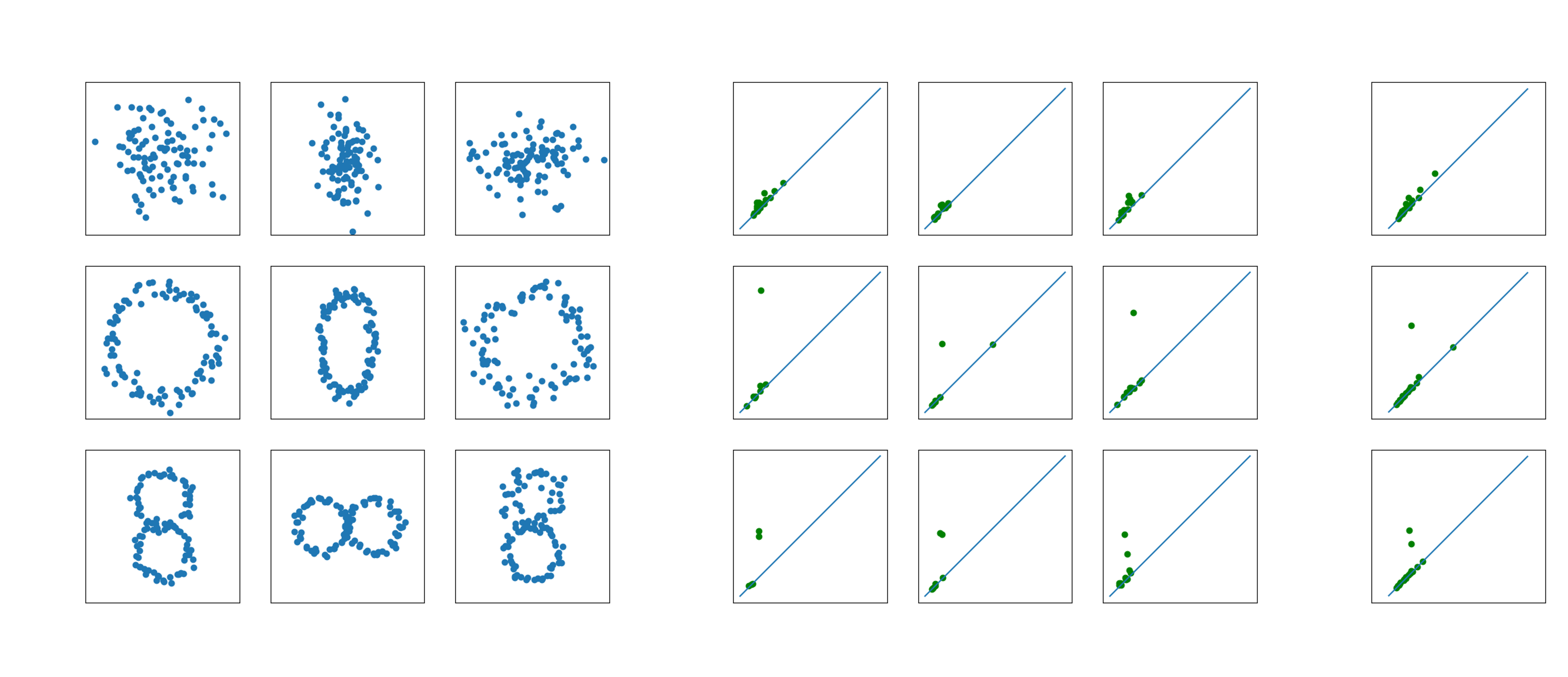}
  \caption{On the left we show nine synthetic datasets, consisting of noise, one hole, or two holes. In the middle we compute the 1-persistence diagrams, which we recall counts the number of holes. We cluster these persistence diagrams, resulting in three learnt cluster centres, shown on the right. The cluster centres have zero, one, and two significantly off-diagonal points: the clustering algorithm has learnt the topological features of the datasets.}
\end{figure*}

\section{Topological preliminaries}

Topological Data Analysis emerged from the study of algebraic topology, providing a toolkit to fully describe the topology of a dataset. We offer a quick summary below; for more comprehensive details see \cite{comptopbook}. A set of points in $\mathbb{R}^d$ are indicative of the shape of the distribution they are sampled from. By connecting points that are pairwise within $\epsilon > 0$ distance of each other, we can create an approximation of the distribution called the Vietoris-Rips complex \citep{rips}. Specifically, we add the convex hull of any collection of points that are pairwise at most $\epsilon$ apart to the $\epsilon$-Vietoris-Rips complex. However, choosing an $\epsilon$ remains problematic; too low a value and key points can remain disconnected, too high a value and the points become fully connected. To overcome this we use \textit{persistence}: we consider the approximation over all values of $\epsilon$ simultaneously, and study how the topology of that approximation evolves as $\epsilon$ grows. We call the collection of complexes for all $\epsilon$ a filtration. One example of a filtration is shown on the left of Figure 1.

For each $\epsilon$, we compute the $p$-homology group. This tells us the topology of the $\epsilon$-Vietoris-Rips complex: the 0-homology counts the number of connected components, the 1-homology counts the number of holes, the 2-homology counts the number of voids, and so on \citep{toppers}. The $p$-persistent homology group is created by summing the $p$-homology groups over all $\epsilon$. This results in a $p$-PH group that summarises information about the topology of the dataset at all granularities. If it only persists for a short amount of time, then it's more likely to be noise \citep{stab_pd}. We can stably map a $p$-PH group into a multiset in the extended plane called a persistence diagram \citep{isometry}. Each topological feature has a birth and death: a feature is born when enough points enter the filtration to create it: in Figure 1 we can see a hole is created in the third complex in the filtration. The topological feature then dies when $\epsilon$ grows large enough to destroy that feature: in the fourth complex the hole is filled in. The birth and death values (i.e., the values of $\epsilon$ when a topological feature enters the filtration or is destroyed) are the axes of the persistence diagram, so each point in the persistence diagram represents a topological feature. In Figure 1, the green point represents the hole. The larger the difference between birth and death values, the longer a topological feature \textit{persists} for, and the more likely the feature is to be a feature of the distribution that the points are sampled from. By computing the birth and death points for each topological feature in the filtration, we get a complete picture of the topology of the point cloud at all granularities \citep{Zomorodian2005}. The persistence diagram is the collection of birth/death points, along with the diagonal $\Delta = \{(a,a) : a\in\mathbb{R}\}$ with infinite multiplicity, added in order to make the space of persistence diagrams complete \citep{milyenko}.

\begin{figure*}[ht!]
\centering
  \includegraphics[width=0.9\textwidth]{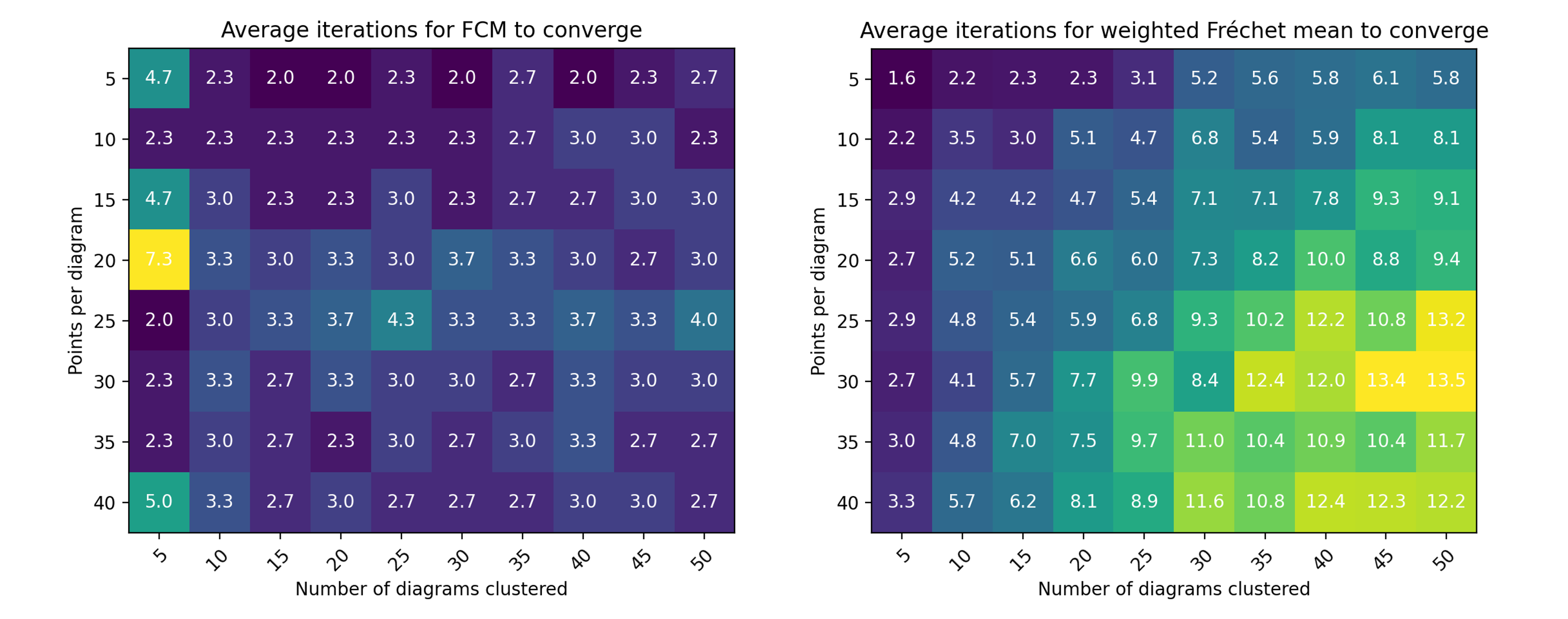}
  \caption{Heatmaps showing average number of iterations for fuzzy clustering of persistence diagrams (left) and the weighted Fréchet mean computation (right) to converge. Convergence of the FCM algorithm is determined when the cost function is stable to within $\pm0.5\%$. Convergence experiments were carried out on a total of 825 persistence diagrams, including three repeats .}
\end{figure*}

\section{Algorithmic design}

\subsection{Clustering persistence diagrams}

In order to cluster we need a distance on the space of persistence diagrams. We use the 2-Wasserstein $L_2$ metric as it is stable on persistence diagrams of finite point clouds \citep{isometry}. The Wasserstein distance is an optimal transport metric that has found applications across machine learning. In the Euclidean case, it quantifies the smallest distance between optimally matched points. Given diagrams $\mathbb{D}_1, \mathbb{D}_2$, the distance is \[W_2(\mathbb{D}_1, \mathbb{D}_2) =  \left( \inf_{\gamma: \mathbb{D}_1 \to \mathbb{D}_2}  \sum_{x \in \mathbb{D}_1} || x - \gamma(x) ||_2^2 \right)^{1/2},\]
where the infimum is taken over all bijections $\gamma: \mathbb{D}_1 \to \mathbb{D}_2$. Note that as we added the diagonal with infinite multiplicity to each diagram, these bijections exist. If an off-diagonal point is matched to the diagonal the transportation cost is simply the shortest distance to the diagonal. In fact, the closer a point is to the diagonal, the more likely it is to be noise \citep{stab_pd}, so this ensures our distance is not overly affected by noise.

We work in the space $\DL2 = \{ \D : W_2(\D, \Delta) < \infty \}$,\footnote{To ensure that our persistence diagrams are all in this space, we map points at infinity to a hyperparameter $T$ that is much larger than other death values in the diagram. Alternatively, this can be avoided entirely by computing the diagrams with extended persistence \citep{extendedpers}, which removes points at infinity.} as this leads to a geodesic space with known structure \citep{Turner2012FrchetMF}. Given a collection of persistence diagrams $\{\mathbb{D}_j\}_{j=1}^n \subset \DL2$ and a fixed number of clusters $c$, we wish to find cluster centres $\{\mathbb{M}_k\}_{k=1}^c \subset \DL2$, along with membership values $r_{jk} \in [0,1]$ that denote the extent to which $\mathbb{D}_j$ is associated with cluster $\mathbb{M}_k$. We follow probabilistic fuzzy clustering, so that $\sum_k r_{jk} = 1$ for each $j$. 

We extend the FCM algorithm originally proposed by \cite{BezdekPaper}. Our $r_{jk}$ is the same as traditional FCM clustering, adapted with the Wasserstein distance. That is, 
\begin{equation}
r_{jk} = \left( \sum_{l=1}^c \frac{ W_2(\M_k, \D_j) }{W_2(\M_l, \D_j)}  \right)^{-1}.
\end{equation}

To update $\mathbb{M}_k$, we compute the weighted Fréchet mean $\hat{\mathbb{D}}$ of the persistence diagrams $\{\mathbb{D}_j\}_{j=1}^n$ with the weights $\{r_{jk}^2\}_{j=1}^n$.

Specifically, 
\begin{equation}
\mathbb{M}_k \longleftarrow \mathrm{arg}\min_{\hat{\mathbb{D}} } \sum_{j=1}^n r_{jk}^2 W_2(\hat{\mathbb{D}}, \mathbb{D}_j)^2, \text{ for } k=1,\dots,c.
\end{equation}

As the weighted Fréchet mean extends weighted centroids to general metric spaces, this gives our fuzzy cluster centres. The computation of the weighted Fréchet mean is covered in Section 3.2. By alternatively updating (1) and (2) we get a sequence of iterates. Theorem 1, proven in Appendix A, provides the same convergence guarantees as traditional FCM clustering.

\begin{thm} 
\label{thm: local convergence}
Every convergent subsequence of the sequence of iterates obtained by alternatively updating membership values and cluster centres with (1) and (2) tends to a local minimum or saddle point of the cost function $J(R, \M) = \sum_{j=1}^n \sum_{k=1}^c r_{jk}^2 W_2(\M_k, \D_j)^2$.
\end{thm}

Observe that we only guarantee the convergence of subsequences of iterates. This is the same as traditional FCM clustering, so we follow the same approach to a stopping condition and run our algorithm for a fixed number of iterations.

\subsection{Computing the weighted Fréchet mean}

\citet{Turner2012FrchetMF} give an algorithm for the computation of Fréchet means. In this section we extend their algorithm and proof of convergence to the weighted case. The proof is by gradient descent, which requires defining a differential structure on the space of persistence diagrams. Our extension to the proof comes down to proving that given some supporting vectors of the Fréchet function, the weighted sum of those is also a supporting vector. For more details see Appendix B.

To give some intuition, start by recalling that when processing the persistence diagrams we add copies of the diagonal to ensure that each diagram has the same cardinality; denote this cardinality as $m$. To compute the weighted Fréchet mean, we need to find $\M_k=\{y^{(i)}\}_{i=1}^m$ that minimises the Fréchet function in (2). Implicit to the Wasserstein distance is a bijection $\gamma_j: y^{(i)} \mapsto x_j^{(i)}$ for each $j$. Supposing we know these bijections, we can rearrange the Fréchet function into the form $F(\M_k) = \sum_{j=1}^n r_{jk}^2 W_2(\M_k, \mathbb{D}_j)^2=\sum_{i=1}^m \sum_{j=1}^n r_{jk}^2 ||y^{(i)} - x_j^{(i)} ||^2.$ In this form, the summand is minimised for $y^{(i)}$ by the weighted Euclidean centroid of the points $\{x_j^{(i)}\}_{j=1}^n$. Therefore to compute the weighted Fréchet mean, we need to find the correct bijections. We start by using the Hungarian algorithm to find an optimal matching between $\mathbb{M}_k$ and each $\mathbb{D}_j$. Given a $\mathbb{D}_j$, for each point $y^{(i)} \in \mathbb{M}_k$, the Hungarian algorithm will assign an optimally matched point $x^{(i)}_j \in \mathbb{D}_j$. Specifically, we find matched points
\begin{equation}
\left[ x_j^{(i)}\right]_{i=1}^m \longleftarrow \text{Hungarian}(\M_k,\D_j), \text{ for each } j=1,\dots,n.
\end{equation}
Now, for each $y^{(i)} \in \M_k$ we need to find the weighted average of the matched points $\left[ x_j^{(i)}\right]_{j=1}^n$. However, some of these points could be copies of the diagonal, so we need to consider three distinct cases: that each matched point is off-diagonal, that each one is a copy of the diagonal, or that the points are a mixture of both. We start by partitioning $1,\dots,n$ into the indices of the off-diagonal points $\mathscr{J}^{(i)}_{\text{OD}} = \left\{ j: x_j^{(i)} \neq \Delta\right\}$ and the indices of the diagonal points $\mathscr{J}^{(i)}_{\text{D}} = \left\{ j: x_j^{(i)} = \Delta\right\}$ for each $i=1,\dots,m$. Now, if $\mathscr{I}_{\text{OD}} = \emptyset$ then $y^{(i)}$ is a copy of the diagonal. If not, let $w = \left(\sum_{j\in \mathscr{J}^{(i)}_{\text{OD}}} r_{jk}^2 \right)^{-1}  \sum_{j\in \mathscr{J}^{(i)}_{\text{OD}}} r_{jk}^2 x^{(i)}_j $ be the weighted mean of the off-diagonal points. If $\mathscr{J}^{(i)}_{\text{D}} = \emptyset$, then $y^{(i)} = w$. Otherwise, let $w_\Delta$ be the point on the diagonal closest to $w$. Then our update is
\begin{equation}
y^{(i)} \longleftarrow \frac{ \sum_{j \in \mathscr{J}^{(i)}_{\text{OD}} } r_{jk}^2 x^{(i)}_j + \sum_{j \in \mathscr{J}^{(i)}_{\text{D}} } r_{jk}^2 w_\Delta }{\sum_{j=1}^n r_{jk}^2}
\end{equation}

for $i=1,\dots,m$. We alternate between (3) and (4) until the matching remains the same. Theorem 2, proving that this algorithm converges to a local minimum of the Fréchet function, is proven in Appendix B. 

\begin{thm} 
\label{thm: weighted fechet mean}
Given diagrams $\D_j$, membership values $r_{jk}$, and the Fréchet function $F(\hat{\D}) = \sum_{j=1}^n r_{jk}^2 W_2(\hat{\D}, \D_j)^2$, then $\M_k = \{ y^{(i)} \}_{i=1}^m$ is a local minimum of $F$ if and only if there is a unique optimal pairing from $\M_k$ to each of the $\D_j$ and each $y^{(i)}$ is updated via (4).
\end{thm}

\section{Experiments}

\subsection{Synthetic data}
\textbf{Exemplar clustering.} We start by demonstrating our algorithm on a simple synthetic dataset designed to highlight its ability to cluster based on the topology of the underlying datasets. We produce three datasets of noise, three datasets of a ring, and three datasets of figure-of-eights, all shown on the left of Figure 2. In the middle of Figure 2 we show the corresponding 1-PH persistence diagrams. Note that the persistence diagrams have either zero, one, or two significant off-diagonal points, corresponding to zero, one, or two holes in the datasets. We then use our algorithm to cluster the nine persistence diagrams into three clusters. Having only been given the list of diagrams, the number of clusters, and the maximum number of iterations, our algorithm successfully clusters the diagrams based on their topology. The right of Figure 2 shows that the cluster centres have zero, one, or two off-diagonal points: our algorithm has found cluster centres that reflect the topological features of the datasets.

\begin{figure}
\centering
  \includegraphics[width=0.5\textwidth]{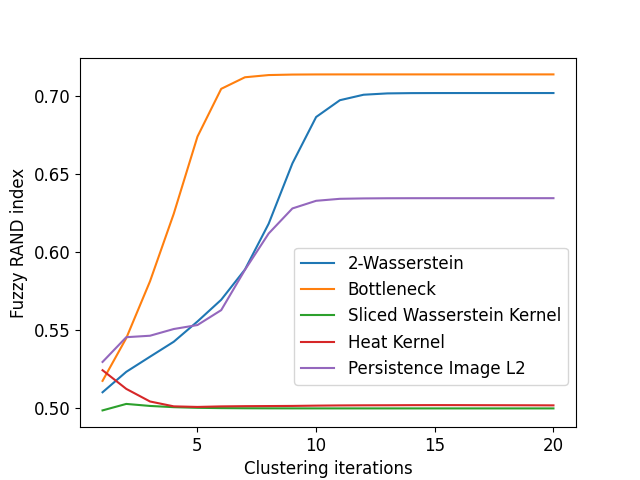}
  \caption{For computational speedups practitioners may wish to use different distances in the clustering algorithm. We use the fuzzy RAND index \citep{frand} to evaluate cluster quality when using some common distances. The more expensive optimal matching-based distances perform best, whereas approximations and embedding-based distances are faster but score lower.}
\end{figure}

\begin{figure*}[ht!]
\centering
  \includegraphics[width=0.99\textwidth]{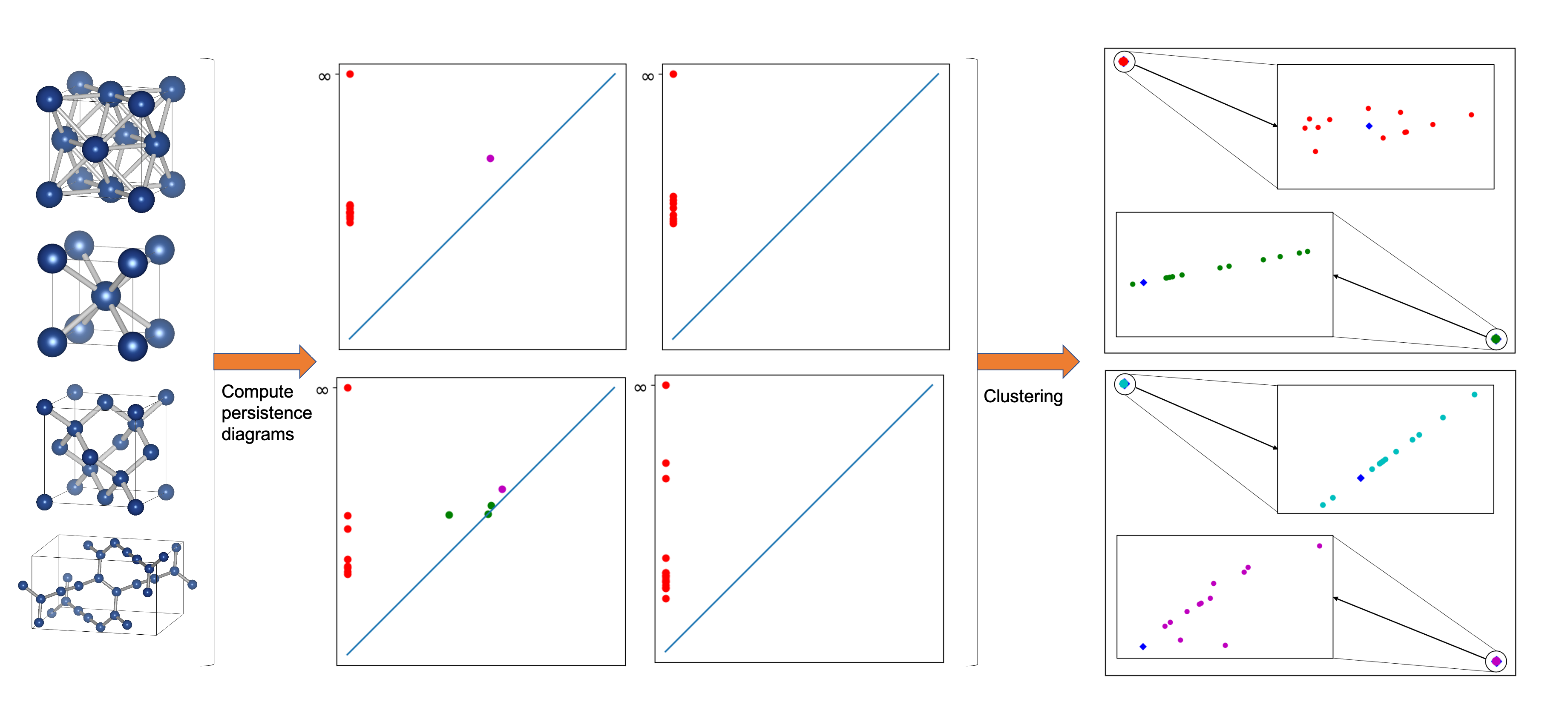}
  \caption{We fuzzy cluster the persistence diagrams of lattice structures, enabling unsupervised classification of transformed materials science datasets for the first time. From left to right we display four examples of the lattice structures, followed by their persistence diagrams (red is 0-PH, green is 1-PH, and magenta is 2-PH). In the final panel we show persistence diagram space embedded into the plane with multidimensional scaling. Each circular point is a persistence diagram, and the diamond points are the learnt cluster centres.}
\end{figure*}

\textbf{Empirical behaviour.} Figure 3 shows the results of experiments run to determine the empirical performance of our algorithm. We give theoretical guarantees that every convergent subsequence will tend to a local minimum, but in practice it remains important that our algorithm will converge within a reasonable timeframe. To this end we ran experiments on a total of 825 randomly generated persistence diagrams, recording the number of iterations and cost functions for both the FCM clustering and the weighted Fréchet mean (WFM) computation. We considered the FCM to have converged when the cost function remained within $\pm 0.5\%$. As explained in Section 3.2, the WFM converges when the matching stays the same. Our experiments showed that the FCM clustering consistently converges within $5$ iterations, regardless of the number of diagrams and points per diagram (note that the time per iteration increases as the number of points/diagrams increases, even if the number of iterations remains stable). We had no experiments in which the algorithm did not converge. The WFM computation requires more iterations as both number of diagrams and number of points per diagram increases, but we once again experienced no failures to converge in each of our experiments. In general, running the algorithm offered no difficulties on an standard laptop,\footnote{A 2018 MacBook Pro with a 1.4 GHz i5 and 8GB RAM.} and we believe the algorithm is ready for use by practitioners in the TDA community.

Because we are reducing the cardinality and dimensionality of datasets by mapping into persistence diagrams, we also demonstrate a speed-up of at least an order of magnitude over Wasserstein barycentre clustering methods. Details of these timing experiments are in Appendix C.1. However, the use of the Wasserstein distance in the clustering still means that some large-scale problems are computationally intractable. To explore solutions to this, we investigated the use of different distances in Equation (1). Specifically, we evaluated the quality of learnt clusters using the fuzzy RAND index \citep{frand} when clustering with the 2-Wasserstein distance, bottleneck distance, sliced Wasserstein kernel \citep{slicekern}, heat kernel, and L2 distance between persistence images \citep{persimages}. We find that the more expensive optimal matching-based distances perform best, whereas approximations and embedding-based distances are faster but score lower. These results are shown in Figure 4.

\subsection{Lattice structures}

Recall that a key property for machine learning in materials science has been identified as ``invariance to the basis symmetries of physics [...] rotation, reflection, translation'' \citep{Schmidt2019}. Removing the need for a standardised coordinate system allows machine learning methods to be applied to a much broader range of existing datasets. Persistence diagrams, which capture affine transformation-invariant properties of datasets, are ideally suited for application in this domain. Our algorithm enables successful unsupervised learning on these datasets for the first time. Additionally, the fuzzy membership values allow top-$k$ ranking of candidates suggested by our algorithm. This is particularly important in materials science, where further investigation of materials can be extremely costly.

% The large majority of solids are comprised of lattices: regularly repeating unit cells of atoms. This lattice structure directly determines the properties of a material \citep{introdis} and it has been predicted that machine learning will reveal presently unknown links between structure and property by identifying new trends across materials \citep{matml2, matml1}.  The most common lattice structures, particularly amongst pure metals, are face-centred cubic (FCC) structures and body-centred cubic (BCC) structures \citep{putnis_1992}. Carbon allotropes, such as graphene and diamond, are widely anticipated to revolutionise electronics and optoelectronics \citep{elec}. We focus on the carbon allotropes diamond and cis-hinged polydiacetylene. Each of these in shown Figure 5.

We analyse the lattice structures of body-centered cubic structures and face-centered cubic structures from the Materials Project \citep{matproject} and two carbon allotropes from the Samara database \citep{sacada}. We simulate distinct collections of lattices by transforming the atomic coordinates, with no information about bonds given to the algorithms. The properties of persistence diagrams mean that we can successfully cluster the atomic coordinates derived from the same base unit-cell regardless of the transformations applied to the coordinate system, fulfilling the key property identified above. In comparison, we run Wasserstein barycentre clustering on the same datasets using several state-of-the-art algorithms for barycentre computation and approximation. Each can only successfully cluster the cubic structures after reflection, and none of them successfully cluster the carbon allotropes after any transformations. Further details on these results are available in Appendix C.2.

% We use atomic positions for the unit-cells of iron mp-150 and iron mp-13 from the Materials Project \citep{matproject}, representing BCC and FCC structures respectively, for our first experiment. For our second experiment we use diamond and cis-hinged polydiacetylene unit-cell atomic positions from the Samara Carbon Allotrope Database \citep{sacada}.

\subsection{Decision boundaries}
Learnt models have been shown to perform better on datasets which have a similar topological complexity to the model's decision boundary \citep{tda_db}. In fact, there is an increasing amount of work studying the link between topology and neural network performance \citep{neurpers, nncap}. We utilise our algorithm to cluster the topology of the decision boundaries of pre-trained models and tasks (labelled datasets). Given a task we find the nearest cluster centre, then select the models nearest to that centre. Even though the only information utilised for the model selection is the topology of the decision boundaries, we find that it consistently selects the top performing model as the first choice, and additional choices perform above average, despite not being trained on the task. This indicates that the topology of the decision boundary is indicative of generalisation ability to unseen tasks. Furthermore, our algorithm is able to exploit this information to learn cluster centres that consistently select the best performing models on tasks.

Specifically, given a dataset with $n$ classes, we fix one class to define $n-1$ \textit{tasks}: binary classification of the fixed class vs each of the remaining classes. On each of these tasks, we train a \textit{model}. We compute the decision boundary of the model $f$, defined as $(x_1, \dots, x_m, f(x))$ where $f(x)$ is the model's prediction for $x=(x_i)_i$, and the decision boundary of the tasks, defined via the labelled dataset as $(x_1, \dots, x_m, y)$ where $y$ is the true label. We compute the $1$-persistence diagrams of the tasks' and models' decision boundaries and cluster them to obtain membership values and cluster centres. To view task and model proximity through our clustering, we find the cluster centre with the highest membership value for each task, and consider the models closest to that cluster centre. Note that model selection is naturally a fuzzy task: one model can (and does) perform well on multiple tasks. Therefore this is a task best suited to fuzzy clustering. 
%We further discuss why this technique does not work for hard clustering in Appendix C.3.
We further discuss why hard clustering does not work here in Appendix C.3.

To assess the ability of our model/task clustering, we performed the above experiment on three different datasets: MNIST \citep{mnist}, FashionMNIST \citep{fashion}, and Kuzushiji-MNIST \citep{kmnist}. We repeat each experiment three times using sequential seeds, resulting in a total of 81 trained models. Our goal is to evaluate whether or not the clustering is capturing information about model performance on tasks, so as a baseline we use the average performance of all models on a fixed task, averaged over all tasks. We start by verifying what happens if we use the model closest to the cluster centre associated with the task (i.e., top-1). We see a significant increase in performance, indicating that the topological fuzzy clustering has selected the model trained on the task, despite only having information about the topology of the decision boundary. We also compute the top-3 and top-2 performance change over average. We still see a statistically significant increase in performance over average performance, indicating that the fuzzy clusters are capturing information about model generalisation to unseen tasks. These results are shown in Figure 6.

\begin{figure}[ht!]
\centering
  \includegraphics[width=0.5\textwidth]{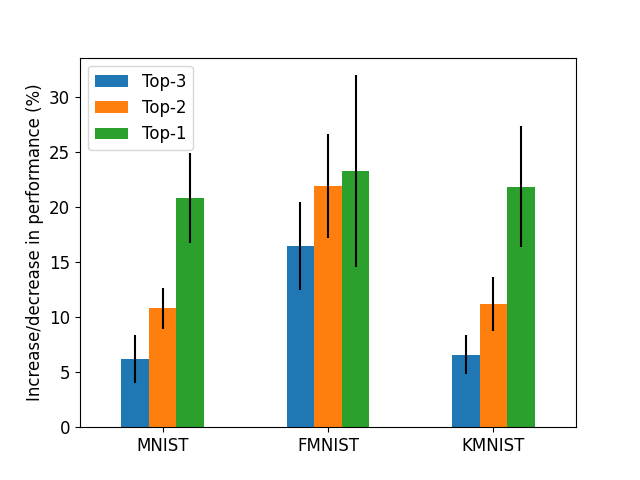}
  \caption{Performance increase/decrease over average task performance when using the fuzzy clustered persistence diagrams of decision boundaries for model selection. Given a task, we find its nearest cluster centre, and use fuzzy membership values to select the nearest models. The increase in performance demonstrates that our fuzzy clustering automatically clusters models near tasks they perform well on, using just the topology of their decision boundaries.}
\end{figure}

\section{Conclusion}

We have developed Fuzzy c-Means clustering on the space of persistence diagrams, adding an important class of unsupervised learning to Topological Data Analysis' toolkit. We give theoretical and empirical convergence results, and study applications to materials science and model selection. We find our results on decision boundaries particularly exciting: that we are able to successfully fuzzy cluster pre-trained models with tasks they generalise well to based only on the topology of their decision boundaries is somewhat surprising. This link between the topology of decision boundaries and their performance on unseen tasks brings further insights to the generalisation power of machine learning.

%\section*{Acknowledgements}

% In the unusual situation where you want a paper to appear in the
% references without citing it in the main text, use \nocite
%\nocite{langley00}

\newpage 

\bibliography{example_paper}

\begin{thebibliography}{62}
\providecommand{\natexlab}[1]{#1}
\providecommand{\url}[1]{\texttt{#1}}
\expandafter\ifx\csname urlstyle\endcsname\relax
  \providecommand{\doi}[1]{doi: #1}\else
  \providecommand{\doi}{doi: \begingroup \urlstyle{rm}\Url}\fi

\bibitem[Bauer(2019)]{ripser}
Bauer, U.
\newblock Ripser: efficient computation of vietoris-rips persistence barcodes,
  August 2019.
\newblock Preprint.

\bibitem[Benamou et~al.(2015)Benamou, Carlier, Cuturi, Nenna, and Peyré]{IBP}
Benamou, J.-D., Carlier, G., Cuturi, M., Nenna, L., and Peyré, G.
\newblock Iterative bregman projections for regularized transportation
  problems.
\newblock \emph{SIAM Journal on Scientific Computing}, 2015.
\newblock \doi{10.1137/141000439}.

\bibitem[{Bezdek}(1980)]{BezdekPaper}
{Bezdek}, J.~C.
\newblock A convergence theorem for the fuzzy isodata clustering algorithms.
\newblock \emph{IEEE Transactions on Pattern Analysis and Machine
  Intelligence}, PAMI-2\penalty0 (1):\penalty0 1--8, Jan 1980.
\newblock ISSN 1939-3539.
\newblock \doi{10.1109/TPAMI.1980.4766964}.

\bibitem[{Bezdek} et~al.(1987){Bezdek}, {Hathaway}, {Sabin}, and
  {Tucker}]{repairs}
{Bezdek}, J.~C., {Hathaway}, R.~J., {Sabin}, M.~J., and {Tucker}, W.~T.
\newblock Convergence theory for fuzzy c-means: Counterexamples and repairs.
\newblock \emph{IEEE Transactions on Systems, Man, and Cybernetics},
  17\penalty0 (5):\penalty0 873--877, 1987.

\bibitem[Bubenik(2015)]{persland}
Bubenik, P.
\newblock Statistical topological data analysis using persistence landscapes.
\newblock \emph{Journal of Machine Learning Research}, 16\penalty0
  (3):\penalty0 77--102, 2015.
\newblock URL \url{http://jmlr.org/papers/v16/bubenik15a.html}.

\bibitem[Bubenik \& Wagner(2019)Bubenik and Wagner]{dist1}
Bubenik, P. and Wagner, A.
\newblock Embeddings of persistence diagrams into hilbert spaces.
\newblock \emph{CoRR}, abs/1905.05604, 2019.
\newblock URL \url{http://arxiv.org/abs/1905.05604}.

\bibitem[Campello(2007)]{frand}
Campello, R.~J.
\newblock A fuzzy extension of the rand index and other related indexes for
  clustering and classification assessment.
\newblock \emph{Pattern Recognition Letters}, 28\penalty0 (7):\penalty0
  833--841, 2007.

\bibitem[Carri{\`e}re \& Bauer(2019)Carri{\`e}re and Bauer]{bilin}
Carri{\`e}re, M. and Bauer, U.
\newblock On the metric distortion of embedding persistence diagrams into
  separable hilbert spaces.
\newblock In \emph{Symposium on Computational Geometry}, 2019.

\bibitem[Carri{\`e}re et~al.(2017)Carri{\`e}re, Cuturi, and Oudot]{slicekern}
Carri{\`e}re, M., Cuturi, M., and Oudot, S.
\newblock Sliced wasserstein kernel for persistence diagrams.
\newblock In \emph{ICML}, 2017.

\bibitem[Carri{\`e}re et~al.(2020)Carri{\`e}re, Chazal, Ike, Lacombe, Royer,
  and Umeda]{perslay}
Carri{\`e}re, M., Chazal, F., Ike, Y., Lacombe, T., Royer, M., and Umeda, Y.
\newblock Perslay: A neural network layer for persistence diagrams and new
  graph topological signatures.
\newblock In \emph{AISTATS}, 2020.

\bibitem[Chazal et~al.(2012)Chazal, Silva, Glisse, and Oudot]{isometry}
Chazal, F., Silva, V., Glisse, M., and Oudot, S.
\newblock \emph{The Structure and Stability of Persistence Modules}.
\newblock 07 2012.
\newblock \doi{10.1007/978-3-319-42545-0}.

\bibitem[Chen et~al.(2018)Chen, Ni, Bai, and Wang]{chenreg}
Chen, C., Ni, X., Bai, Q., and Wang, Y.
\newblock Toporeg: {A} topological regularizer for classifiers.
\newblock \emph{CoRR}, abs/1806.10714, 2018.
\newblock URL \url{http://arxiv.org/abs/1806.10714}.

\bibitem[Chepushtanova et~al.(2015)Chepushtanova, Emerson, Hanson, Kirby,
  Motta, Neville, Peterson, Shipman, and Ziegelmeier]{persimages}
Chepushtanova, S., Emerson, T., Hanson, E., Kirby, M., Motta, F., Neville, R.,
  Peterson, C., Shipman, P., and Ziegelmeier, L.
\newblock Persistence images: An alternative persistent homology
  representation.
\newblock 07 2015.

\bibitem[Clanuwat et~al.(2018)Clanuwat, Bober-Irizar, Kitamoto, Lamb, Yamamoto,
  and Ha]{kmnist}
Clanuwat, T., Bober-Irizar, M., Kitamoto, A., Lamb, A., Yamamoto, K., and Ha,
  D.
\newblock Deep learning for classical japanese literature, 2018.

\bibitem[{Clough} et~al.(2020){Clough}, {Byrne}, {Oksuz}, {Zimmer}, {Schnabel},
  and {King}]{seg5}
{Clough}, J., {Byrne}, N., {Oksuz}, I., {Zimmer}, V.~A., {Schnabel}, J.~A., and
  {King}, A.
\newblock A topological loss function for deep-learning based image
  segmentation using persistent homology.
\newblock \emph{IEEE Transactions on Pattern Analysis and Machine
  Intelligence}, pp.\  1--1, 2020.

\bibitem[Cohen-Steiner et~al.(2007)Cohen-Steiner, Edelsbrunner, and
  Harer]{stab_pd}
Cohen-Steiner, D., Edelsbrunner, H., and Harer, J.
\newblock Stability of persistence diagrams.
\newblock \emph{Discrete {\&} Computational Geometry}, 37\penalty0
  (1):\penalty0 103--120, Jan 2007.
\newblock ISSN 1432-0444.
\newblock \doi{10.1007/s00454-006-1276-5}.
\newblock URL \url{https://doi.org/10.1007/s00454-006-1276-5}.

\bibitem[Cohen-Steiner et~al.(2009)Cohen-Steiner, Edelsbrunner, and
  Harer]{extendedpers}
Cohen-Steiner, D., Edelsbrunner, H., and Harer, J.
\newblock Extending persistence using poincare and lefschetz duality.
\newblock \emph{FOUNDATIONS OF COMPUTATIONAL MATHEMATICS}, pp.\  2009, 2009.

\bibitem[Cuturi \& Doucet(2014)Cuturi and Doucet]{sinkhorn}
Cuturi, M. and Doucet, A.
\newblock Fast computation of wasserstein barycenters.
\newblock volume~32 of \emph{Proceedings of Machine Learning Research}, pp.\
  685--693, Bejing, China, 22--24 Jun 2014. PMLR.
\newblock URL \url{http://proceedings.mlr.press/v32/cuturi14.html}.

\bibitem[d.~A.~T.~{de Carvalho} et~al.(2015)d.~A.~T.~{de Carvalho}, {Irpino},
  and {Verde}]{fwbc}
d.~A.~T.~{de Carvalho}, F., {Irpino}, A., and {Verde}, R.
\newblock Fuzzy clustering of distribution-valued data using an adaptive l2
  wasserstein distance.
\newblock In \emph{2015 IEEE International Conference on Fuzzy Systems
  (FUZZ-IEEE)}, pp.\  1--8, 2015.

\bibitem[Edelsbrunner \& Harer(2010)Edelsbrunner and Harer]{comptopbook}
Edelsbrunner, H. and Harer, J.
\newblock \emph{Computational Topology - an Introduction.}
\newblock American Mathematical Society, 2010.
\newblock ISBN 978-0-8218-4925-5.

\bibitem[Edelsbrunner et~al.(2000)Edelsbrunner, Letscher, and
  Zomorodian]{toppers}
Edelsbrunner, H., Letscher, D., and Zomorodian, A.
\newblock Topological persistence and simplification.
\newblock volume~28, pp.\  454 -- 463, 02 2000.
\newblock ISBN 0-7695-0850-2.
\newblock \doi{10.1109/SFCS.2000.892133}.

\bibitem[Fabio \& Ferri(2015)Fabio and Ferri]{fabio1}
Fabio, B.~D. and Ferri, M.
\newblock Comparing persistence diagrams through complex vectors.
\newblock In \emph{Image Analysis and Processing {\textemdash} {ICIAP} 2015},
  pp.\  294--305. Springer International Publishing, 2015.
\newblock \doi{10.1007/978-3-319-23231-7_27}.
\newblock URL \url{https://doi.org/10.1007/978-3-319-23231-7_27}.

\bibitem[Gabrielsson et~al.(2020)Gabrielsson, Nelson, Dwaraknath, and
  Skraba]{toplayer}
Gabrielsson, R.~B., Nelson, B.~J., Dwaraknath, A., and Skraba, P.
\newblock A topology layer for machine learning.
\newblock volume 108 of \emph{Proceedings of Machine Learning Research}, pp.\
  1553--1563, Online, 26--28 Aug 2020. PMLR.
\newblock URL \url{http://proceedings.mlr.press/v108/gabrielsson20a.html}.

\bibitem[Gidea \& Katz(2018)Gidea and Katz]{crash}
Gidea, M. and Katz, Y.
\newblock Topological data analysis of financial time series: Landscapes of
  crashes.
\newblock \emph{Physica A: Statistical Mechanics and its Applications},
  491:\penalty0 820--834, February 2018.
\newblock \doi{10.1016/j.physa.2017.09.028}.
\newblock URL \url{https://doi.org/10.1016/j.physa.2017.09.028}.

\bibitem[Guss \& Salakhutdinov(2018)Guss and Salakhutdinov]{nncap}
Guss, W.~H. and Salakhutdinov, R.
\newblock On characterizing the capacity of neural networks using algebraic
  topology.
\newblock \emph{ArXiv}, abs/1802.04443, 2018.

\bibitem[Hofer et~al.(2017)Hofer, Kwitt, Niethammer, and Uhl]{hofer1}
Hofer, C., Kwitt, R., Niethammer, M., and Uhl, A.
\newblock Deep learning with topological signatures.
\newblock In Guyon, I., Luxburg, U.~V., Bengio, S., Wallach, H., Fergus, R.,
  Vishwanathan, S., and Garnett, R. (eds.), \emph{Advances in Neural
  Information Processing Systems 30}, pp.\  1634--1644. Curran Associates,
  Inc., 2017.
\newblock URL
  \url{http://papers.nips.cc/paper/6761-deep-learning-with-topological-signatures.pdf}.

\bibitem[Hoffmann et~al.(2016)Hoffmann, Kabanov, Golov, and Proserpio]{sacada}
Hoffmann, R., Kabanov, A., Golov, A., and Proserpio, D.
\newblock Homo citans and carbon allotropes: For an ethics of citation.
\newblock \emph{Angewandte Chemie International Edition}, 55, 07 2016.
\newblock \doi{10.1002/anie.201600655}.

\bibitem[Jain et~al.(2013)Jain, Ong, Hautier, Chen, Richards, Dacek, Cholia,
  Gunter, Skinner, Ceder, and Persson]{matproject}
Jain, A., Ong, S.~P., Hautier, G., Chen, W., Richards, W.~D., Dacek, S.,
  Cholia, S., Gunter, D., Skinner, D., Ceder, G., and Persson, K.~a.
\newblock {The Materials Project: A materials genome approach to accelerating
  materials innovation}.
\newblock \emph{APL Materials}, 1\penalty0 (1):\penalty0 011002, 2013.
\newblock ISSN 2166532X.
\newblock \doi{10.1063/1.4812323}.
\newblock URL \url{http://link.aip.org/link/AMPADS/v1/i1/p011002/s1\&Agg=doi}.

\bibitem[Kali{\v{s}}nik(2018)]{tropical}
Kali{\v{s}}nik, S.
\newblock Tropical coordinates on the space of persistence barcodes.
\newblock \emph{Foundations of Computational Mathematics}, 19\penalty0
  (1):\penalty0 101--129, January 2018.
\newblock \doi{10.1007/s10208-018-9379-y}.
\newblock URL \url{https://doi.org/10.1007/s10208-018-9379-y}.

\bibitem[Kerber et~al.(2017)Kerber, Morozov, and Nigmetov]{geompd}
Kerber, M., Morozov, D., and Nigmetov, A.
\newblock Geometry helps to compare persistence diagrams.
\newblock \emph{Journal of Experimental Algorithmics (JEA)}, 22:\penalty0
  1--20, 2017.

\bibitem[Kim et~al.(2020)Kim, Kim, Kim, Chazal, and Wasserman]{landlay}
Kim, K., Kim, J., Kim, J., Chazal, F., and Wasserman, L.
\newblock Efficient topological layer based on persistent landscapes.
\newblock \emph{ArXiv}, abs/2002.02778, 2020.

\bibitem[Klawonn(2004)]{fuzzy}
Klawonn, F.
\newblock Fuzzy clustering: Insights and a new approach.
\newblock \emph{Mathware \& soft computing, ISSN 1134-5632, Vol. 11, Nº. 3,
  2004, pags. 125-142}, 11, 01 2004.

\bibitem[Kovacev-Nikolic et~al.(2014)Kovacev-Nikolic, Bubenik, Nikolić, and
  Heo]{protein}
Kovacev-Nikolic, V., Bubenik, P., Nikolić, D., and Heo, G.
\newblock Using persistent homology and dynamical distances to analyze protein
  binding.
\newblock Statistical Applications in Genetics and Molecular Biology. January
  2016, Volume 15, Issue 1, Pages 19-38, 2014.

\bibitem[Lacombe et~al.(2018)Lacombe, Cuturi, and Oudot]{lacombe}
Lacombe, T., Cuturi, M., and Oudot, S.
\newblock Large scale computation of means and clusters for persistence
  diagrams using optimal transport.
\newblock In Bengio, S., Wallach, H., Larochelle, H., Grauman, K.,
  Cesa-Bianchi, N., and Garnett, R. (eds.), \emph{Advances in Neural
  Information Processing Systems 31}, pp.\  9770--9780. Curran Associates,
  Inc., 2018.

\bibitem[Le \& Yamada(2018)Le and Yamada]{fishkern}
Le, T. and Yamada, M.
\newblock Persistence fisher kernel: A riemannian manifold kernel for
  persistence diagrams.
\newblock In \emph{NeurIPS}, 2018.

\bibitem[LeCun et~al.(2010)LeCun, Cortes, and Burges]{mnist}
LeCun, Y., Cortes, C., and Burges, C.
\newblock Mnist handwritten digit database.
\newblock \emph{ATT Labs [Online]. Available:
  http://yann.lecun.com/exdb/mnist}, 2, 2010.

\bibitem[{Li} \& {Lewis}(2016){Li} and {Lewis}]{fuzzyreview}
{Li}, J. and {Lewis}, H.~W.
\newblock Fuzzy clustering algorithms — review of the applications.
\newblock In \emph{2016 IEEE International Conference on Smart Cloud
  (SmartCloud)}, pp.\  282--288, 2016.
\newblock \doi{10.1109/SmartCloud.2016.14}.

\bibitem[{Li} \& {Wang}(2008){Li} and {Wang}]{linprog}
{Li}, J. and {Wang}, J.~Z.
\newblock Real-time computerized annotation of pictures.
\newblock \emph{IEEE Transactions on Pattern Analysis and Machine
  Intelligence}, 30\penalty0 (6):\penalty0 985--1002, 2008.

\bibitem[{Ling} \& {Okada}(2007){Ling} and {Okada}]{emd}
{Ling}, H. and {Okada}, K.
\newblock An efficient earth mover's distance algorithm for robust histogram
  comparison.
\newblock \emph{IEEE Transactions on Pattern Analysis and Machine
  Intelligence}, 29\penalty0 (5):\penalty0 840--853, 2007.

\bibitem[Maroulas et~al.(2017)Maroulas, Mike, and Marchese]{inproceedings}
Maroulas, V., Mike, J., and Marchese, A.
\newblock K-means clustering on the space of persistence diagrams.
\newblock In \emph{SPIE}, pp.\ ~29, 08 2017.
\newblock \doi{10.1117/12.2273067}.

\bibitem[Mileyko et~al.(2011)Mileyko, Mukherjee, and Harer]{milyenko}
Mileyko, Y., Mukherjee, S., and Harer, J.
\newblock Probability measures on the space of persistence diagrams.
\newblock \emph{Inverse Problems - INVERSE PROBL}, 27, 12 2011.
\newblock \doi{10.1088/0266-5611/27/12/124007}.

\bibitem[Moor et~al.(2019)Moor, Horn, Rieck, and Borgwardt]{topauto}
Moor, M., Horn, M., Rieck, B.~A., and Borgwardt, K.
\newblock Topological autoencoders.
\newblock \emph{ArXiv}, abs/1906.00722, 2019.

\bibitem[{Pantula} et~al.(2020){Pantula}, {Miriyala}, {Giri}, and
  {Mitra}]{fuzzy2}
{Pantula}, P.~D., {Miriyala}, S.~S., {Giri}, L., and {Mitra}, K.
\newblock Synchronicity identification in hippocampal neurons using artificial
  neural network assisted fuzzy c-means clustering.
\newblock In \emph{2020 IEEE Symposium Series on Computational Intelligence
  (SSCI)}, pp.\  1594--1600, 2020.
\newblock \doi{10.1109/SSCI47803.2020.9308344}.

\bibitem[Ramamurthy et~al.(2019)Ramamurthy, Varshney, and Mody]{tda_db}
Ramamurthy, K.~N., Varshney, K., and Mody, K.
\newblock Topological data analysis of decision boundaries with application to
  model selection.
\newblock volume~97 of \emph{Proceedings of Machine Learning Research}, pp.\
  5351--5360, Long Beach, California, USA, 09--15 Jun 2019. PMLR.
\newblock URL \url{http://proceedings.mlr.press/v97/ramamurthy19a.html}.

\bibitem[{Reininghaus} et~al.(2015){Reininghaus}, {Huber}, {Bauer}, and
  {Kwitt}]{stabkern}
{Reininghaus}, J., {Huber}, S., {Bauer}, U., and {Kwitt}, R.
\newblock A stable multi-scale kernel for topological machine learning.
\newblock In \emph{2015 IEEE Conference on Computer Vision and Pattern
  Recognition (CVPR)}, pp.\  4741--4748, 2015.

\bibitem[Rieck et~al.(2018)Rieck, Togninalli, Bock, Moor, Horn, Gumbsch, and
  Borgwardt]{neurpers}
Rieck, B., Togninalli, M., Bock, C., Moor, M., Horn, M., Gumbsch, T., and
  Borgwardt, K.~M.
\newblock Neural persistence: {A} complexity measure for deep neural networks
  using algebraic topology.
\newblock \emph{CoRR}, abs/1812.09764, 2018.
\newblock URL \url{http://arxiv.org/abs/1812.09764}.

\bibitem[Rieck et~al.(2019)Rieck, Sadlo, and Leitte]{rieck3}
Rieck, B.~A., Sadlo, F., and Leitte, H.
\newblock Topological machine learning with persistence indicator functions.
\newblock \emph{ArXiv}, abs/1907.13496, 2019.

\bibitem[Schmidt et~al.(2019)Schmidt, Marques, Botti, and Marques]{Schmidt2019}
Schmidt, J., Marques, M. R.~G., Botti, S., and Marques, M. A.~L.
\newblock {Recent advances and applications of machine learning in solid-state
  materials science}.
\newblock \emph{npj Computational Materials}, 5\penalty0 (1):\penalty0 83,
  2019.
\newblock ISSN 2057-3960.
\newblock \doi{10.1038/s41524-019-0221-0}.
\newblock URL \url{https://doi.org/10.1038/s41524-019-0221-0}.

\bibitem[Turner \& Spreemann(2019)Turner and Spreemann]{pdisometry}
Turner, K. and Spreemann, G.
\newblock Same but different: Distance correlations between topological
  summaries.
\newblock \emph{arXiv: Algebraic Topology}, 2019.

\bibitem[Turner et~al.(2012)Turner, Mileyko, Mukherjee, and
  Harer]{Turner2012FrchetMF}
Turner, K., Mileyko, Y., Mukherjee, S., and Harer, J.
\newblock Fr{\'e}chet means for distributions of persistence diagrams.
\newblock \emph{Discrete \& Computational Geometry}, 52:\penalty0 44--70, 2012.

\bibitem[{Vidal} et~al.(2020){Vidal}, {Budin}, and {Tierny}]{pb}
{Vidal}, J., {Budin}, J., and {Tierny}, J.
\newblock Progressive wasserstein barycenters of persistence diagrams.
\newblock \emph{IEEE Transactions on Visualization and Computer Graphics},
  26\penalty0 (1):\penalty0 151--161, 2020.

\bibitem[Vietoris(1927)]{rips}
Vietoris, L.
\newblock Über den höheren zusammenhang kompakter röume und eine klasse von
  zusammenhangstreuen abbildungen.
\newblock \emph{Mathematische Annalen}, 97\penalty0 (1):\penalty0 454--472,
  December 1927.
\newblock \doi{10.1007/bf01447877}.
\newblock URL \url{https://doi.org/10.1007/bf01447877}.

\bibitem[Wagner(2019)]{dist2}
Wagner, A.
\newblock Nonembeddability of persistence diagrams with p>2 wasserstein metric.
\newblock \emph{ArXiv}, abs/1910.13935, 2019.

\bibitem[Wilder et~al.(2019)Wilder, Ewing, Dilkina, and Tambe]{bryan}
Wilder, B., Ewing, E., Dilkina, B., and Tambe, M.
\newblock End to end learning and optimization on graphs.
\newblock In \emph{NeurIPS}, 2019.

\bibitem[Xiao et~al.(2017)Xiao, Rasul, and Vollgraf]{fashion}
Xiao, H., Rasul, K., and Vollgraf, R.
\newblock Fashion-mnist: a novel image dataset for benchmarking machine
  learning algorithms, 2017.

\bibitem[{Yang} et~al.(2019){Yang}, {Yao}, {Yu}, {Lee}, and {Zhang}]{ex1}
{Yang}, H., {Yao}, Q., {Yu}, A., {Lee}, Y., and {Zhang}, J.
\newblock Resource assignment based on dynamic fuzzy clustering in elastic
  optical networks with multi-core fibers.
\newblock \emph{IEEE Transactions on Communications}, 67\penalty0 (5):\penalty0
  3457--3469, 2019.
\newblock \doi{10.1109/TCOMM.2019.2894711}.

\bibitem[{Ye} \& {Li}(2014){Ye} and {Li}]{ADMM}
{Ye}, J. and {Li}, J.
\newblock Scaling up discrete distribution clustering using admm.
\newblock In \emph{2014 IEEE International Conference on Image Processing
  (ICIP)}, pp.\  5267--5271, 2014.

\bibitem[Ye et~al.(2017)Ye, Wu, Wang, and Li]{BADDM}
Ye, J., Wu, P., Wang, J., and Li, J.
\newblock Fast discrete distribution clustering using wasserstein barycenter
  with sparse support.
\newblock \emph{IEEE Transactions on Signal Processing}, PP:\penalty0 1--1, 01
  2017.
\newblock \doi{10.1109/TSP.2017.2659647}.

\bibitem[Zangwill(1969)]{zangwill1969nonlinear}
Zangwill, W.
\newblock \emph{Nonlinear programming: a unified approach}.
\newblock Prentice-Hall international series in management. Prentice-Hall,
  1969.
\newblock URL \url{https://books.google.co.uk/books?id=TWhxLcApH9sC}.

\bibitem[Zhao \& Wang(2019)Zhao and Wang]{learnweight}
Zhao, Q. and Wang, Y.
\newblock Learning metrics for persistence-based summaries and applications for
  graph classification.
\newblock In Wallach, H., Larochelle, H., Beygelzimer, A., d'~Alch\'{e}-Buc,
  F., Fox, E., and Garnett, R. (eds.), \emph{Advances in Neural Information
  Processing Systems 32}, pp.\  9859--9870. Curran Associates, Inc., 2019.

\bibitem[Zieliński et~al.(2019)Zieliński, Lipiński, Juda, Zeppelzauer, and
  Dłotko]{persbag}
Zieliński, B., Lipiński, M., Juda, M., Zeppelzauer, M., and Dłotko, P.
\newblock Persistence bag-of-words for topological data analysis.
\newblock In \emph{Proceedings of the Twenty-Eighth International Joint
  Conference on Artificial Intelligence, {IJCAI-19}}, pp.\  4489--4495.
  International Joint Conferences on Artificial Intelligence Organization, 7
  2019.
\newblock \doi{10.24963/ijcai.2019/624}.
\newblock URL \url{https://doi.org/10.24963/ijcai.2019/624}.

\bibitem[Zomorodian \& Carlsson(2005)Zomorodian and Carlsson]{Zomorodian2005}
Zomorodian, A. and Carlsson, G.
\newblock {Computing Persistent Homology}.
\newblock \emph{Discrete {\&} Computational Geometry}, 33\penalty0
  (2):\penalty0 249--274, 2005.
\newblock ISSN 1432-0444.
\newblock \doi{10.1007/s00454-004-1146-y}.
\newblock URL \url{https://doi.org/10.1007/s00454-004-1146-y}.

\end{thebibliography}
\bibliographystyle{icml2020}

% DELETE THIS PART. DO NOT PLACE CONTENT AFTER THE REFERENCES!

%\end{document}
\newpage

\appendix
\section{Convergence of the FCM clustering algorithm}
We first need to consider the update steps (1) and (2) as a single update procedure. Let $F:\M \mapsto R$ be defined by (1) and $G:R\mapsto\M$ be defined by (2), and for $R = \{r_{jk}\}$ and $\M = \{\M_k\}$ consider the sequence \[\left\{ T^{(l)}(R, \M) : l=0,1,\dots \right\}\] where $T(R,M) = (F\circ G(R), G(R))$. We wish to show convergence of the iterates of $T$ to a local minimum or saddle point of the cost function \[J(R, \M) = \sum_{j=1}^n \sum_{k=1}^c r_{jk}^2 W_2(\M_k, \D_j)^2.\] The two stage update process of $T$ is too complicated to use standard fixed point theorems, so following \citet{BezdekPaper} we shall use the following result, which is proven by \citet{zangwill1969nonlinear}.

\begin{thm}[Zangwill's Convergence Theorem] Let $A: X \to 2^X$ be a point-to-set algorithm acting on $X$. Given $x_0\in X$, generate a sequence $\{x_k\}_{k=1}^\infty$ such that $x_{k+1} \in A(x_k)$ for every $k$. Let $\Gamma \subset X$ be a solution set, and suppose that the following hold.
\begin{enumerate}[(i)]
\item The sequence $\{x_k\} \subset S \subset X$ for a compact set $S$.
\item There exists a continuous function $Z$ on $X$ such that if $x\not\in \Gamma$ then $Z(y)<Z(x)$ for all $y\in A(x)$, and if $x\in \Gamma$ then $Z(y)\leq Z(x)$ for all $y\in A(x)$. The function $Z$ is called a descent function.
\item The algorithm $A$ is closed on $X\setminus \Gamma$.
\end{enumerate}
Then every convergent subsequence of $\{x_k\}$ tends to a point in the solution set $\Gamma$.
\end{thm}

Our algorithm is the update function $T$. We define our solution set as 
\begin{align*}
\Gamma = \Big\{ (R, \M) : J(R,\M) < J(\hat{R}, \hat{\M}) \\ \forall \ (\hat{R}, \hat{\M}) \in B( (R, \M), r)  \Big\}
\end{align*}
for some $r>0$, where the ball surrounding $R$ is the Euclidean ball in $\mathbb{R}^{nc}$ and the ball surrounding $\M$ is $\cup_{k=1}^c B_{W_2}(\M_k, r)$. This set contains the local minima and saddle points of the cost function \citep{repairs}. We wish to show that our cost function $J(R, \M)$ is the descent function $Z$. We proceed by verifying each of the requirements for Zangwill's Convergence Theorem.

\begin{lemma} Every iterate $T^{(l)}(R, \M) \in [0,1]^{nc} \times \mathrm{conv}(\D)^c$, where \[\mathrm{conv}(\D) = \bigcup_{k=1}^c \bigcup_{\gamma_j} \bigcup_{i=1}^m \mathrm{conv}\{\gamma_j(y^{(i)}): j=1,\dots,n\},\] with $\gamma_j$ a bijection $\M_k \to \D_j$ and $\mathrm{conv}\{\gamma_j(y^{(i)}): j= 1,\dots,n\}$ the ordinary convex hull in the plane. Furthermore, $[0,1]^{nc} \times \mathrm{conv}(\D)^c$ is compact.
\end{lemma}

\begin{proof}By construction, every $r_{jk} \in [0,1]$. Since $j=1,\dots,n$ and $k=1,\dots,c$, we can view $R$ as a point in $[0,1]^{nc}$, and so every iterate of $R$ is in $[0,1]^{nc}$. We shall show that for a fixed $k$ and a fixed bijection $\gamma_j:\M_k \to \D_j$, each updated $y^{(i)}$ is contained in a convex combination of $\{\gamma_j(y^{(i)}):j=1,\dots,n\}$. Where $\gamma_j(y^{(i)}) = \Delta$, let $\gamma_j(y^{(i)}) = w_\Delta$ as defined in (4), as this is the update point we use for the diagonal. Since there are a finite number of off-diagonal points, each updated $\M_k$ is therefore contained in the union over all bijections and all points $y^{(i)}$ of the convex combination of $\{\gamma_j(y^{(i)}):j=1,\dots,n\}$. By also taking the union over each $k$, we show that every iterate of $\M$ must be contained in the finite triple-union of the convex combination of each possible bijection. To show that each updated $y^{(i)}$ is contained in a convex combination of $\{\gamma_j(y^{(i)}):j=1,\dots,n\}$, recall that $y^{(i)} = \left(\sum_{j=1}^n r_{jk}^2\right)^{-1} \sum_{j=1}^n r_{jk}^2 \gamma_j(y^{(i)})$. Letting $t_j^{(i)} =  r_{jk}^2 \left(\sum_{j=1}^n r_{jk}^2\right)^{-1}$, clearly each $t_j^{(i)} >0$ and $\sum_{j=1}^n t_j^{(i)} = 1$. Since $y^{(i)} = \sum_{j=1}^n t_j^{(i)} \gamma_j(y^{(i)})$, each $y^{(i)}$ is contained in the convex combination. Therefore $T^{(l)}(R, \M) \in [0,1]^{nc} \times \mathrm{conv}(\D)^c$ for each $l=0,1,\dots$.

Now, $[0,1]$ is closed and bounded, so is compact. The convex hull of points in the plane is closed and bounded, so $\mathrm{conv}\{\gamma_j(y^{(i)}):j=1,\dots,n\}$ is compact. Since finite unions and finite direct products of compact sets are compact, $[0,1]^{nc} \times \mathrm{conv}(\D)^c$ is also compact.
\end{proof}

\begin{lemma} The cost function $J(R, \M)$ is a descent function, as defined in Theorem 3(ii).
\end{lemma}

\begin{proof}The cost function $J$ is continuous, as it's a sum,  product and composition of continuous functions. Furthermore, we have that for any $(R, \M) \not\in \Gamma$,
\begin{align*}
J(T(R,\M)) &= J(F\circ G(R), G(R)) \\
&< J(R, G(R)) \\ 
&< J(R, M),
\end{align*}
where the first inequality is due to Proposition 1 in \cite{BezdekPaper}, and the second inequality comes from the definition of the Fréchet mean. If $(R, \M) \in \Gamma$ then the strict inequalities include equality throughout.
\end{proof}

\begin{thm} For any $(R, \M)$, every convergent subsequence of $\{T^{(l)}(R, \M) : l = 0,1,\dots\}$ tends to a local minimum or saddle point of the cost function $J$.
\end{thm}

\begin{proof} We proceed with Zangwill's Convergence Theorem. Our algorithm is the update function $T$, our solution set is $\Gamma$, and our descent function is the cost function $J(R, \M)$. By Lemma 4, every iterate is contained within a compact set. By Lemma 5, $J$ is a descent function. Finally, since our function $T$ only maps points in the plane to points in the plane, it is a closed map. The theorem follows by applying Theorem 3.
\end{proof}

\section{Convergence of the Fréchet mean algorithm}
Recall that the Fréchet mean is computed by finding the $\text{arg}\min$ of 
\begin{equation}
F(\hat{\D}) = \sum_{j=1}^n r_{jk}^2 F_j( \hat{\D}), \text{ with } F_j(\hat{\D}) = W_2(\hat{\D}, \D_j)^2,
\end{equation}
 for fixed $k$. We start by recounting work by \citet{Turner2012FrchetMF}, which this section adapts for the weighted Fréchet mean.\footnote{In \citet{Turner2012FrchetMF}, the Fréchet mean is defined as the arg min of the Fréchet function $F(\hat{\D}) = \int W_2(\hat{\D}, \D_j)^2 d\rho(\hat{\D})$ with the empirical measure $\rho = n^{-1} \sum_{j=1}^n \delta_{\D_j}$. We are using the empirical measure $\rho = \left(\sum_{j=1}^n r_{jk}^2 \right)^{-1} \sum_{j=1}^n r_{jk}^2 \delta_{\D_j}$, but for ease we drop the scalar $\left(\sum_{j=1}^n r_{jk}^2 \right)^{-1}$ as it is positive, so it does not affect the minimum of the function.} The proofs we're adapting use a gradient descent technique to prove local convergence. In order to use their techniques, we need to define a differential structure on the space of persistence diagrams.

By Theorem 2.5 from \citet{Turner2012FrchetMF}, the space of persistence diagrams $\mathscr{D}_{L^2} = \left\{ \D : W_2(\D, \Delta) < \infty \right\}$ is a non-negatively curved Alexandrov space. An optimal bijection $\gamma:\D_1 \to \D_2$ induces a unit-speed geodesic $\phi(t) = \{(1-t)x + t\gamma(x) : x\in \D_1, 0\leq t \leq 1 \}$. For a point $\D \in \DL2$ we define the tangent cone $T_{\D}$. Define $\hat{\Sigma}_{\D}$ as the set of all non-trivial unit-speed geodesics emanating from $\D$.  Let $\phi,\eta \in \hat{\Sigma}_{\D}$ and define the angle between them as \[\angle_{\D}(\phi, \eta) = \arccos\left( \lim_{s,t\downarrow 0} \frac{s^2 + t^2 - W_2( \phi(s), \eta(t))^2 }{2st}  \right)\] in $ [0,\pi]$ when the limit exists. Then the space of directions $(\Sigma_{\D}, \angle_{\D})$ is the completion of $\hat{\Sigma}_{\D} / \sim$ with respect to $\angle_{\D}$, with $\phi \sim \eta \iff \angle_{\D}(\phi, \eta) = 0$. We now define the tangent cone as  \[T_{\D} = (\Sigma_{\D} \times [0,\infty))/(\Sigma_{\D} \times \{0\}).\] Given $u=(\phi, s), v=(\eta,t)$, we define an inner product on the tangent cone by \[\langle u, v\rangle = st \cos \angle_{\D}(\phi, \eta).\] Now, for $\alpha>0$ denote the space $(\DL2, \alpha W_2)$ as $\alpha\DL2$ and define the map $i_\alpha: \alpha\DL2 \to \DL2$. For an open set $\Omega \subset \DL2$ and a function $f:\Omega \to \mathbb{R}$, the differential of $f$ at $\D \in \Omega$ is defined by $d_{\D} f = \lim_{\alpha\to\infty} \alpha(f \circ i_{\D} - f(\D))$. Finally, we say that $s \in T_{\D}$ is a supporting vector of $f$ at $\D$ if $d_{\D}f(x) \leq -\langle s,x\rangle$ for all $x \in T_{\D}$.

\begin{lemma} The following two results are proven in \citet{Turner2012FrchetMF}.
\begin{enumerate}[(i)]
\item Let $\D\in \DL2$. Let $F_{j}(\hat{\D}) = W_2(\hat{\D}, \D_j)^2$. Then if $\phi$ is a distance-achieving geodesic from $\D$ to $\hat{D}$, then the tangent vector to $\phi$ at $\D$ of length $2W_2(\hat{\D}, \D)$ is a supporting vector at $\D$ of $f$.
\item If $\D$ is a local minimum of $f$ and $s$ is a supporting vector of $f$ at $\D$, then $s=0$. 
\end{enumerate}
\end{lemma}

If there is a unique optimal matching $\gamma_{\D_1}^{\D_3} : \D_1 \to \D_3$, we say that it is induced by an optimal matching $\gamma_{\D_1}^{\D_2} : \D_1 \to \D_2$ if there exists a unique optimal matching  $\gamma_{\D_2}^{\D_3} : \D_2 \to \D_3$ such that $\gamma_{\D_1}^{\D_3} = \gamma_{\D_2}^{\D_3} \circ \gamma_{\D_1}^{\D_2}$. Proposition 3.2 from \citet{Turner2012FrchetMF} states that an optimal matching at a point is also locally optimal. In particular, it states the following.

\begin{lemma}Let $\D_1,\D_2\in \DL2$ such that there is a unique optimal matching from $\D_1$ to $\D_2$. Then there exists an $r>0$ such that for every $\D_3 \in B_{W_2}(\D_2, r)$, there is a unique optimal pairing from $\D_2$ to $\D_3$ that is induced by the matching from $\D_1$ to $\D_2$. 
\end{lemma}

The following theorem proves that our algorithm converges to a local minimum of the Fréchet function.

\begin{thm} Given diagrams $\D_j$, membership values $r_{jk}$, and the Fréchet function $F$ defined in (5), then $\M_k = \{ y^{(i)} \}_{i=1}^m$ is a local minimum of $F$ if and only if there is a unique optimal pairing from $\M_k$ to each of the $\D_j$, denoted $\gamma_j$, and each $y^{(i)}$ is updated via (4).
\end{thm}

\begin{proof} First assume that $\gamma_j$ are optimal pairings from $\M_k$ to each $\D_j$, and let $s_j$ be the vectors in $T_{\M_k}$ that are tangent to the geodesics induced by $\gamma_j$ and are distance-achieving. Then by Lemma 7(i), each $2s_j$ is a supporting vector for the function $F_j$. Furthermore, $2\sum_{j=1}^n r_{jk}^2 s_j$ is a supporting vector for $F$, as for any $\hD$,
\begin{align*}
d_{\M_k} F(\hat{\D}) &= d_{\M_k}\left( \sum_{j=1}^n r_{jk}^2 F_j(\hat{\D}) \right) =  \sum_{j=1}^n r_{jk}^2 d_{\M_k} F_j(\hat{\D}) \\ &\leq \sum_{j=1}^n - r_{jk}^2 \langle2s_j, \hD \rangle = - \left\langle 2 \sum_{j=1}^n r_{jk}^2 s_j, \hD \right\rangle.
\end{align*}
By Lemma 7(ii), $ 2 \sum_{j=1}^n r_{jk}^2 s_j = 0$. Putting $s_j =  \gamma_j(y^{(i)}) - y^{(i)}$ and rearranging gives that $y^{(i)}$ updates via (4), as required. Note that when $\gamma_j(y^{(i)}) = \Delta$, we let $\gamma_j(y^{(i)}) = w_\Delta$ as defined in (4), because this minimises the transportation cost to the diagonal. Now suppose that $\gamma_j$ and $\tilde{\gamma}_j$ are both optimal pairings. Then by the above argument $ \sum_{j=1}^n r_{jk}^2 s_j =  \sum_{j=1}^n r_{jk}^2 \tilde{s}_j = 0$, implying that $s_j = \tilde{s}_j$ and so $\gamma_j = \tilde{\gamma}_j$. Therefore the optimal pairing is unique.

To prove the opposite direction, assume that $\M_k = \{y^{(i)}\}$ locally minimises the Fréchet function $F$. Observe that for a fixed bijection $\gamma_j$, we have that

\begin{align*}
F(\M_k) &= \sum_{j=1}^n r_{jk}^2 W_2(\M_k, \mathbb{D}_j)^2\\
&= \sum_{j=1}^n r_{jk}^2 \left(\inf_{\gamma_j: \hat{M} \to \D_j}  \sum_{y \in \M_k} || y - \gamma_j(y) ||^2\right)\\
&= \sum_{j=1}^n r_{jk}^2 \sum_{i=1}^m ||y^{(i)} - x_j^{(i)} ||^2\\
&=\sum_{i=1}^m \left(\sum_{j=1}^n r_{jk}^2 ||y^{(i)} - x_j^{(i)} ||^2\right).
\end{align*}

The final term in brackets is non-negative, and minimised exactly when $y^{(i)}$ is updated via (4).  Furthermore, the unique optimal pairing from $\M_k$ to each of the $\D_j$'s is the same for every $\hat{\M}$ within the ball $B_{W_2}(\M_k, r)$ for some $r>0$, by Lemma 8. Therefore, if $\M_k$ is a local minimum of $F$, then the $y^{(i)}$'s are equal to the values found by taking the optimal pairings $\gamma_j$ and calculating the weighted means of $\gamma_j(y^{(i)})$ with the weights $r_{jk}^2$, as required. It will remain a minimum as long as the matching stays the same, which happens in the ball $B_{W_2}(\M_k, r)$, so we are done.
\end{proof}

\section{Experimental details}

\subsection{Synthetic data}

\textbf{Membership values.} The membership values for the synthetic datasets are in Table 1. Datasets 1-3 are the datasets of noise, datasets 4-6 are the datasets with one ring, and datasets 7-9 are the datasets with two rings. We ran our algorithm for 20 iterations. The code to generate the dataset is available in the supplementary materials.

\begin{table*}[ht]
\centering
\tabcolsep=0.15cm
\caption{Membership values for the synthetic dataset\linebreak}
\begin{tabular}{lccccccccc}
\toprule
Dataset & 1 & 2 & 3 & 4 & 5 & 6 & 7 & 8 & 9 \\\midrule 
Cluster 1 & 0.6336 & 0.5730 & 0.5205 & 0.2760 & 0.2503 & 0.1974 & 0.2921 & 0.2128 & 0.2292 \\
Cluster 2 & 0.1768 & 0.2057 & 0.2327 & 0.5361 & 0.5329 & 0.6371 & 0.2452 & 0.2291 & 0.1822 \\
Cluster 3 & 0.1900 & 0.2212 & 0.2468 & 0.1879 & 0.2169 & 0.1655 & 0.4627 & 0.5580 & 0.5885 \\ \bottomrule
\end{tabular}
\end{table*}

\textbf{Timing experiments.} For the timing experiments we divide the total number of points equally between four distributions, two of which are noise and two of which are shaped in a ring. Each clustering algorithm was run for five iterations on one core of a 2018 MacBook Pro with a 1.4GHz Intel Core i5. We included the time taken to compute the persistence diagrams in the running times for our algorithm.

We also use synthetic data to empirically compare the running time of our algorithm to other dataset clustering algorithms available. Computing the Wasserstein distance has super-cubic time complexity \cite{emd}, so is a significant bottleneck both for our algorithm and comparable Wasserstein barycentre clustering algorithms  \cite{IBP, sinkhorn, linprog, ADMM, BADDM}. Persistence diagrams generally reduce both the dimensionality and number of points in a dataset,\footnote{Persistence diagrams are always planar, so if the data is in $\mathbb{R}^d, d>2,$ then there is a dimensionality reduction. For $p>0$, the persistence diagram of $p$-PH always has less points than the dataset when computed with the Vietoris-Rips complex.} so we in turn reduce the computational bottleneck. To demonstrate this, we evaluated the average time per iteration of our persistence diagram clustering algorithm, as well as the average iteration time for comparable Wasserstein barycentre clustering algorithms. We included the time taken to compute the persistence diagrams from the datasets when timing our clustering algorithm. We give the results in Table 2, leaving an entry blank where it became unpractical to run a test (e.g. it takes too long to return a solution and the algorithm becomes unresponsive). We show at least an order of magnitude improvement in performance over comparable Wasserstein barycentre clustering algorithms.

\textbf{Empirical performance experiments.} We get empirical results on the convergence of a total of 825 randomly generated persistence diagrams. Following Euclidean fuzzy clustering, we denote convergence when the cost function is stable to within $\pm0.5\%$. The WFM converges when the matching remains stable, which we proves does happen. We use the seeds 0, 1, and 2 respectively for our repeats. 

We implement the Fuzzy RAND index \citep{frand}, available in the supplementary materials. We use the same synthetic dataset as before to evaluate our cluster quality, with the origin of the data (noise, one hole, or two holes) as a reference partition. We used Persim\footnote{\url{https://persim.scikit-tda.org/en/latest/}} to compute the additional distances.

\begin{table*}
\centering
\tabcolsep=0.16cm
\caption{Seconds per clustering iteration\linebreak}
\begin{tabular}{lcccccccccc}
\toprule
%\multicolumn{5}{|c|}{Succesfully clustered carbon allotropes?}
Points & 100 & 200 & 300 & 400 & 500 & 600 & 700 & 800 & 900 & 1000 \\ \midrule

\textbf{FPDCluster}              & \textbf{0.01552} & \textbf{0.1975} & \textbf{0.9358} & \textbf{2.229} & \textbf{5.694} & \textbf{12.29} & \textbf{19.27} & \textbf{34.50} & \textbf{53.20} & \textbf{77.81} \\
ADMM                   & 5.622 & 34.86 & 161.3 & 617.6 & - & - & - & - & - & - \\
BADMM                  & 0.2020 & 2.188 & 26.38 & 112.6 & - & - & - & - & - & - \\
SubGD                  & 0.4217 & 2.273 & 22.17 & 103.4 & - & - & - & - & - & - \\
IterBP                 & 0.3825 & 2.226 & 21.57 & 108.9 & - & - & - & - & - & - \\
LP                     & 0.3922 & 2.031 & 22.32 & 117.3 & - & - & - & - & - & - \\\bottomrule
\end{tabular}
\end{table*}

\subsection{Lattice structures}
 
Table 3 gives an overview of the clustering results, with specific values available in Tables 4-7. The fuzzy values for FPDCluster are given as floats, although in each case they converged to an absolute cluster. The Wasserstein barycentre clustering algorithms each have discrete labels. The correct labellings are for 1-3 and 4-6 to be clustered together in each case. We clustered the 2-PH diagrams. We denote a label as having been assigned by 1, or not assigned by 0. We ran each algorithm for five iterations. We obtained our datasets as cif files, converted them to xyz files, and then to csv files, producing a list of the coordinates of each atom in $\mathbb{R}^3$. We create three copies of each structure. For rotation, we rotate two of them by $180^{\mathrm{o}}$ around different axes. For reflection, we reflect two of them in different axes. For translation, we translate them up or down by the length of the unit-cell. We use our own python implementation of FPDCluster, available in the supplementary materials. For each of the other algorithms, we use the implementation provided at \url{https://github.com/bobye/WBC_Matlab}, a copy of which is in the supplementary materials. We do not limit the number of points in the diagram when clustering.

\subsection{Decision boundaries}

\textbf{Why hard clustering does not work.} In order to assign each task to the top-ranked models, we need to have a path from a task to the nearest cluster centre, then from that cluster centre to the $k$-nearest models (note that when we refer to models/tasks, we're implicitly referring to the persistence diagram of their decision boundary). We can always find that route when fuzzy clustering, as the fractional membership values mean that we have information about the proximity of every model/task with every cluster centre. However, with hard clustering we cannot always find that route. Firstly, the hard labelling means that you lose a lot of information about the proximity of models/tasks to cluster centres. Therefore, in order to find a route, we need a every task to be assigned to a cluster centre that also has a model assigned to it. However, there are no guarantees that will happen. We show an example where no path exists in Figure 7.

\begin{figure*}
\centering
\includegraphics[width=1\linewidth]{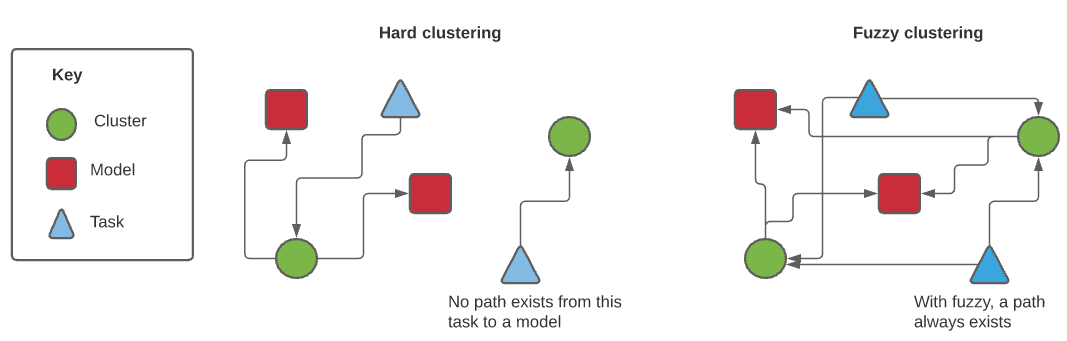}  
\caption{With hard clustering, we cannot always find a path from a task to a model.}
\label{fig:allotropes}
\end{figure*}

\textbf{Experimental details.} All code used for computation is available in the supplementary materials. For models, we trained the standard Pytorch CNN available at \url{https://github.com/pytorch/examples/blob/master/mnist/main.py}. We trained them on MNIST, FashionMNIST, and KMNIST, each obtained using the Torchvision.datasets package. We split the data into 9 binary datasets for classification, class 0 vs each of the remaining classes. We trained three of each model, seeded with 0, 1, and 2 respectively. MNIST and KMNIST were each trained for five epochs, FashionMNIST was trained for 14 epochs. Our train:test split was 6:1, as is standard for MNIST structured datasets. We used Ripser to compute the $1$-persistence diagrams using default settings. We limited the number of points in the diagram to the 25 most persistent when clustering. Our percentage improvement values use the membership values after 16 iterations. We compute the standard error bounds when calculating the percentage improvement.
\begin{table*}[ht]
\centering
\caption{Clustering results after transformation\linebreak}
\begin{tabular}{lccccccccc}
\toprule
 & \multicolumn{4}{c}{Cubic Structures} & & \multicolumn{4}{c}{Carbon Allotropes} \\ \cline{2-5} \cline{7-10}
 & None & Rotate & Reflect & Translate & & None & Rotate & Reflect & Translate \\ \midrule 
\textbf{FPDCluster}               & \bcmark  & \bcmark       & \bcmark     & \bcmark  & & \bcmark  & \bcmark       & \bcmark     & \bcmark       \\
ADMM                    & \cmark  & \xmark        & \cmark      & \xmark  &&\cmark&\xmark&\xmark&\xmark       \\
BADMM                   & \cmark  & \xmark        & \cmark      & \xmark  &&\cmark&\xmark&\xmark&\xmark       \\
SubGD                   & \cmark  & \xmark        & \cmark      & \xmark  &&\cmark&\xmark&\xmark&\xmark       \\
IterBP                  & \cmark  & \xmark        & \cmark      & \xmark  &&\cmark&\xmark&\xmark&\xmark       \\
LP                      & \cmark  & \xmark        & \cmark      & \xmark  &&\cmark&\xmark&\xmark&\xmark       \\\bottomrule
\end{tabular}
\end{table*}
 
\begin{table*}[ht]
\centering
\tabcolsep=0.09cm
\caption{Membership values for non-transformed datasets\linebreak}
\small
\begin{tabular}{llccccccccccccc}
\toprule
 & &\multicolumn{6}{c}{Cubic Structure Datasets} & & \multicolumn{6}{c}{Carbon Allotrope Datasets} \\
\cline{3-8}\cline{10-15}
&  & 1 & 2 & 3 & 4 & 5 & 6 & & 1 & 2 & 3 & 4 & 5 & 6 \\\midrule 
FPDCluster & Cluster 1 &1.000&1.000&1.000&0.000&0.000&0.000&&1.000&1.000&1.000&0.000&0.000&0.000 \\
& Cluster 2 &0.000&0.000&0.000&1.000&1.000&1.000&&0.000&0.000&0.000&1.000&1.000&1.000  \\ \midrule 
ADMM & Cluster 1 &1 &1&1&0&0&0&&1 &1&1&0&0&0 \\
& Cluster 2 &0&0&0&1&1&1&&0&0&0&1&1&1  \\ \midrule 
BADMM & Cluster 1 &1 &1&1&0&0&0&&1 &1&1&0&0&0 \\
& Cluster 2 &0&0&0&1&1&1&&0&0&0&1&1&1  \\ \midrule 
SubGD & Cluster 1 &1 &1&1&0&0&0&&1 &1&1&0&0&0 \\
& Cluster 2 &0&0&0&1&1&1&&0&0&0&1&1&1  \\ \midrule 
IterBP & Cluster 1 &1 &1&1&0&0&0&&1 &1&1&0&0&0 \\
& Cluster 2 &0&0&0&1&1&1&&0&0&0&1&1&1  \\ \midrule 
LP & Cluster 1 &1 &1&1&0&0&0&&1 &1&1&0&0&0 \\
& Cluster 2 &0&0&0&1&1&1&&0&0&0&1&1&1  \\ \bottomrule 
\end{tabular}
\end{table*}

\begin{table*}[ht]
\centering
\tabcolsep=0.09cm
\caption{Membership values for rotated datasets\linebreak}
\small
\begin{tabular}{llccccccccccccc}
\toprule
 & &\multicolumn{6}{c}{Cubic Structure Datasets} & & \multicolumn{6}{c}{Carbon Allotrope Datasets} \\
\cline{3-8}\cline{10-15}
&  & 1 & 2 & 3 & 4 & 5 & 6 & & 1 & 2 & 3 & 4 & 5 & 6 \\\midrule 
FPDCluster & Cluster 1 &1.000&1.000&1.000&0.000&0.000&0.000&&1.000&1.000&1.000&0.000&0.000&0.000 \\
& Cluster 2 &0.000&0.000&0.000&1.000&1.000&1.000&&0.000&0.000&0.000&1.000&1.000&1.000  \\ \midrule 
ADMM & Cluster 1 & 0&0&1&0&0&1&&1&0&1&1&0&1 \\
& Cluster 2 &1&1&0&1&1&0&&0&1&0&0&1&0 \\ \midrule 
BADMM & Cluster 1 &0&1&1&0&1&1&&1&1&0&1&1&0 \\
& Cluster 2 &1&0&0&1&0&0&&0&0&1&0&0&1  \\ \midrule 
SubGD & Cluster 1 &0&1&1&0&1&1&&1&0&0&1&0&0 \\
& Cluster 2 &1&0&0&1&0&0&&0&1&1&0&1&1  \\ \midrule 
IterBP & Cluster 1 &0&1&0&0&1&0&&0&1&1&0&1&1 \\
& Cluster 2 &1&0&1&1&0&1&&1&0&0&1&0&0 \\ \midrule 
LP & Cluster 1 & 1&0&1&1&0&1&&0&1&0&0&1&0 \\
& Cluster 2 &0&1&0&0&1&0&&1&0&1&1&0&1 \\\bottomrule
\end{tabular}
\end{table*}

\begin{table*}[ht]
\centering
\tabcolsep=0.09cm
\caption{Membership values for reflected datasets\linebreak}
\small
\begin{tabular}{llccccccccccccc}
\toprule
 & &\multicolumn{6}{c}{Cubic Structure Datasets} & & \multicolumn{6}{c}{Carbon Allotrope Datasets} \\
\cline{3-8}\cline{10-15}
&  & 1 & 2 & 3 & 4 & 5 & 6 & & 1 & 2 & 3 & 4 & 5 & 6 \\\midrule 
FPDCluster & Cluster 1 &1.000&1.000&1.000&0.000&0.000&0.000&&1.000&1.000&1.000&0.000&0.000&0.000 \\
& Cluster 2 &0.000&0.000&0.000&1.000&1.000&1.000&&0.000&0.000&0.000&1.000&1.000&1.000  \\ \midrule 
ADMM & Cluster 1 & 1&1&1&0&0&0&&0&0&0&1&1&0 \\
& Cluster 2 &0&0&0&1&1&1&&1&1&1&0&0&1  \\ \midrule 
BADMM & Cluster 1 & 1&1&1&0&0&0&&0&0&0&1&1&0 \\
& Cluster 2 &0&0&0&1&1&1&&1&1&1&0&0&1  \\ \midrule 
SubGD & Cluster 1 & 1&1&1&0&0&0&&1&1&1&1&1&0 \\
& Cluster 2 &0&0&0&1&1&1&&0&0&0&0&0&1  \\ \midrule 
IterBP & Cluster 1 & 1&1&1&0&0&0&&0&0&0&0&0&1 \\
& Cluster 2 &0&0&0&1&1&1&&1&1&1&1&1&0 \\ \midrule 
LP & Cluster 1 & 1&1&1&0&0&0&&0&0&0&0&1&1 \\
& Cluster 2 &0&0&0&1&1&1&&1&1&1&1&0&0  \\ \bottomrule 
\end{tabular}
\end{table*}

\begin{table*}[ht]
\centering
\tabcolsep=0.09cm
\caption{Membership values for translated datasets\linebreak}
\small
\begin{tabular}{llccccccccccccc}
\toprule
 & &\multicolumn{6}{c}{Cubic Structure Datasets} & & \multicolumn{6}{c}{Carbon Allotrope Datasets} \\
\cline{3-8}\cline{10-15}
&  & 1 & 2 & 3 & 4 & 5 & 6 & & 1 & 2 & 3 & 4 & 5 & 6 \\\midrule 
FPDCluster & Cluster 1 &1.000&1.000&1.000&0.000&0.000&0.000&&1.000&1.000&1.000&0.000&0.000&0.000 \\
& Cluster 2 &0.000&0.000&0.000&1.000&1.000&1.000&&0.000&0.000&0.000&1.000&1.000&1.000  \\ \midrule 
ADMM & Cluster 1 &0&0&1&0&0&1&&0&0&1&0&0&1 \\
& Cluster 2 &1&1&0&1&1&0&&1&1&0&1&1&0  \\ \midrule 
BADMM & Cluster 1 &0&1&0&1&1&0&&1&1&0&1&1&0 \\
& Cluster 2 &1&0&1&0&0&1&&0&0&1&0&0&1  \\ \midrule 
SubGD & Cluster 1 & 0&0&1&0&0&1&&0&1&0&0&1&0 \\
& Cluster 2 &1&1&0&1&1&0&&1&0&1&1&0&1  \\ \midrule 
IterBP & Cluster 1 & 0&1&0&0&1&0&&0&0&1&0&0&1 \\
& Cluster 2 &1&0&1&1&0&1&&1&1&0&1&1&0  \\ \midrule 
LP & Cluster 1 &0&0&1&0&0&1&&1&0&1&1&0&1 \\
& Cluster 2 &1&1&0&1&1&0&&0&1&0&0&1&0 \\\bottomrule
\end{tabular}
\end{table*}

\end{document}